\documentclass{article}
\PassOptionsToPackage{table}{xcolor}
\usepackage{times}
\usepackage{soul}
\usepackage{url}
\usepackage{cite}
\usepackage[hidelinks]{hyperref}
\usepackage[utf8]{inputenc}
\usepackage[small]{caption}
\usepackage{todonotes}
\usepackage{graphicx}
\usepackage{amsmath}
\usepackage{amsthm}
\usepackage{booktabs}
 \usepackage{subcaption}
\usepackage{tikz,pgfplots}
\usepackage{epsfig}
\usepackage{amsmath}	
\usepackage{amsfonts}
\usepackage{amssymb}
\usepackage{hyperref}
\usepackage{array}
\usepackage{siunitx}
\usepackage{wrapfig}
\usepackage{lipsum}
\newcolumntype{C}{>$c<$}
\usepackage{graphicx}
\usepackage{mathtools}
\usepackage[lined, boxed, ruled, commentsnumbered, noend, linesnumbered]{algorithm2e}
\usepackage{enumerate}

\usepackage{color,soul}
\providecommand{\U}[1]{\protect\rule{.1in}{.1in}}                                                                                           
\newtheorem{theorem}{Theorem}

\newtheorem{corollary}{Corollary}
                                                  
\newtheorem{definition}{Definition}

\newtheorem{lemma}{Lemma}

\DeclareMathOperator*{\dom}{dom}

\let\oldIEEEkeywords\IEEEkeywords
\def\IEEEkeywords{\oldIEEEkeywords\normalfont\bfseries\ignorespaces}
\newcommand{\argmin}{\mathrm{arg}\min}

\newcommand{\denselist}{
  \itemsep -2pt\topsep-5pt\partopsep-5pt\itemindent-10pt
}

\newcommand{\figref}[1]{Fig.~\ref{#1}}

\newcommand{\secref}[1]{\S\ref{#1}}

\newcommand{\algref}[1]{Algorithm~\ref{#1}}

\newcommand{\Hypotheses}{\Phi}

\newcommand{\hstar}{\hypothesis^*}
\newcommand{\hinit}{{\hypothesis^0}}

\newcommand{\demonstrations}{D}

\newcommand{\demonstration}{{z}}

\newcommand{\prefset}{\mathcal{P}}
\newcommand{\ordering}{\sigma}

\newcommand{\orderingof}[2]{\ordering({#1};{#2})}

\def \argmin {\mathop{\rm arg\,min}}

\newcommand{\hypothesis}[0]{\ensuremath{\phi}}
\numberwithin{equation}{section}

\usepackage{subcaption}
\usepackage{wrapfig}

\begin{document}
	
\title{Adaptive Teaching of Temporal Logic Formulas to Learners with Preferences}

\author{Zhe~Xu\thanks{Zhe~Xu is with the Oden Institute
		for Computational Engineering and Sciences, University of Texas,
		Austin, Austin, TX 78712, Yuxin Chen is with the Department of Computer Science,
		University of Chicago, Chicago, IL 60637, Ufuk Topcu is with the Department
		of Aerospace Engineering and Engineering Mechanics, and the Oden Institute
		for Computational Engineering and Sciences, University of Texas,
		Austin, Austin, TX 78712, e-mail: zhexu@utexas.edu,  chenyuxin@uchicago.edu, utopcu@utexas.edu.}, Yuxin Chen, and Ufuk Topcu
}

% Machine
\date{}
\maketitle

\begin{abstract}                                                                          
Machine teaching is an algorithmic framework for teaching a target hypothesis via a sequence of examples or demonstrations. We investigate machine teaching for \textit{temporal logic formulas}---a novel and expressive hypothesis class amenable to time-related \textit{task specifications}. %Finding the optimal teaching sequence with minimal \textit{teaching cost} is notoriously hard. In fact, 
In the context of teaching temporal logic formulas, an exhaustive search even for a \emph{myopic} solution takes exponential time (with respect to the time span of the task). 
We propose an efficient approach for teaching \textit{parametric linear temporal logic formulas}. Concretely, we derive a necessary condition for the minimal time length of a demonstration to eliminate a set of hypotheses.
%, which %allows fast elimination of inconsistent hypothesis by a demonstration. 
%effectively narrows down the space of demonstrations. 
Utilizing this condition, we propose a myopic teaching algorithm by solving a sequence of integer programming problems. We further show that, under two notions of \textit{teaching complexity}, the proposed algorithm has near-optimal performance. The results strictly generalize the previous results on teaching preference-based \emph{version space learners}. We evaluate our algorithm extensively under a variety of learner types (i.e., learners 
%\yuxin{should we go a global replacement from ``learner'' to ``learner''? perhaps we can keep it in the abstract, but replacing in main paper? I agree. I will do that.} 
with different preference models) and interactive protocols (e.g., batched and adaptive). The results show that the proposed algorithms can efficiently teach a given target temporal logic formula under various settings, and that there are significant gains of teaching efficacy when the teacher adapts to the learner's current hypotheses or uses \emph{oracles}.

\end{abstract}    	

% We define \textit{accumulate-number (AN) teaching cost} and \textit{accumulate-length (AL) teaching cost} to measure the minimal number of demonstrations and the minimal accumulated time lengths of the demonstrations needed for the learner to uniquely select the target hypothesis, respectively. 

%novel temporal logic-based interactive machine teaching
%framework, which combines inference and machine teaching in sequential decision-making 
%tasks. 
	
\section{Introduction}                                   
	\label{sec_intro}
	Machine teaching, also known as algorithmic teaching, is an algorithmic framework for teaching a target hypothesis via a
	sequence of examples or demonstrations \cite{Zhu2015}. Due to limited availability of data or high cost of data collection in real-world learning scenarios, machine teaching provides a viable solution to optimize the training data for a learner to efficiently learn a target hypothesis.
	
	Recently, there has been an increasing interest in designing learning algorithms for inferring \textit{task specifications} from data (e.g., in robotics) \cite{Kong2017TAC,VazquezChanlatte2018LearningTS}. Machine teaching can be used to optimize the training data for various learning algorithms \cite{dasgupta2019teaching}, hence it can be potentially used for task specification inference algorithms. Machine teaching can be also used in adversarial settings \cite{ma2019policy} where an attacker (machine teaching algorithm) manipulates specification inference by modifying the training data. Unfortunately, finding the optimal teaching sequence with minimal \textit{teaching cost} is notoriously hard \cite{goldman1995complexity}. As task specifications are often time-related and the demonstration data are trajectories with a time evolution, an exhaustive search even for the \emph{myopic} solution has exponential time complexity with the time span of the task, making it prohibitive to run existing myopic teaching algorithms in practice.

In this paper, we investigate machine teaching of target hypothesis represented in \textit{temporal logic} \cite{Pnueli}, which has been used to express task specifications in many applications in robotics and artificial intelligence \cite{Kress2011,SonThanh_MTL}. Specifically, we use a fragment of \textit{parametric linear temporal logic} (pLTL) \cite{pLTL2014,pLTL}.
 %The complexity of computing the optimal demonstrations for minimizing the \textit{teaching costs} is exponential with the size of the ground set of demonstrations. 
 We derive a necessary condition for the minimal time length of a demonstration so that a set of pLTL formulas can be eliminated by this demonstration. Utilizing this necessary condition, we provide a \textit{myopic teaching} approach by solving a sequence of integer programming problems which, under certain conditions, guarantees a logarithmic factor of the optimal teaching cost. 
 
%  We further investigate the gain of adaptive teaching, and teaching using \textit{lookahead oracles} for reducing the teaching costs. 
 
%  We also provide conditions for teaching temporal logic formulas with positive demonstrations only.

% We implement the machine teaching approach on learners with uniform preferences, (non-uniform) global preferences and local preferences. The results show that the proposed approach can efficiently teach target hypothesis to a learner.
% The results show that the adaptive teaching can reduce the teaching cost by up to 31.15\% in comparison with non-adaptive teaching in the presence of uncertainties. The teaching with oracles can further reduce the teaching cost by up to 75\%.

We evaluate the proposed algorithm extensively under a variety of learner types (i.e., learner with different preference models) and interactive protocols (i.e., batched and adaptive). The results show that the proposed algorithm can efficiently teach a given target hypothesis under various settings, and that there are significant gains of teaching efficacy when the teacher adapts to the learner's current hypotheses (up to 31.15\% reduction in teaching cost compared to the non-adaptive setting) or with oracles (up to 75\% reduction in teaching cost compared to the myopic setting).

\vspace{0.1in}

\noindent\textbf{Related Work}	

%\noindent\textbf{Machine teaching to version space learners} Our work is closely related to the machine teaching of concept classes \cite{hunziker2019teaching,chen18adaptive}. Current work use simple concepts such as rectangles or lattices, or disjunctive normal form formulas. This work considers temporal logic formulas which are more expressive in expressing concepts with a time evolution.

% \paragraph{Algorithmic machine teaching}
%\noindent\textbf{Concept classes for algorithmic machine teaching} 
There has been a surge of interest in machine teaching in several different application domains, including personalized educational systems \cite{Zhu2015}, citizen sciences \cite{mac2018teaching}, adversarial attacks \cite{ma2019policy} and imitation learning \cite{brown2019machine}. %Most existing work consider simple concept classes, such as linear classifiers or regressors, simple geometrical objects, or disjunctive normal form formulas. In contrast, this work considers temporal logic formulas which are more expressive in expressing concepts with a time evolution.
% %Our work is closely related to the machine teaching of concept classes \cite{hunziker2019teaching,chen18adaptive}. Current work use simple concepts such as rectangles or lattices, or disjunctive normal form formulas. This work considers temporal logic formulas which are more expressive in expressing concepts with a time evolution...
% reinforcement learning-based methods \cite{haug2018teaching, brown2019machine}: simple hypothesis classes, e.g., linear regressors/classifiers; 
%
%\noindent\textbf{Learner models in machine teaching} 
Most theoretical work in algorithmic machine teaching assumes the version space model \cite{goldman1995complexity,gao2017preference,chen18adaptive,mansouri2019preference}. %a learner maintains a subset of hypotheses that are consistent with the examples received from a teacher and outputs a hypothesis from this version space. 
Recently, some teaching complexity results have been extended beyond version space learners, such as Bayesian learners \cite{zhu2013machine} %randomized learners \cite{balbach2011teaching}, 
and gradient learners \cite{liu2017iterative} (e.g., learners implementing a gradient-based optimization algorithm). %Within the model class of version space learner, preference-based models \cite{chen18adaptive,} has gained much attention due to its connection with the natural notion of incremental learning in the real-world scenarios, for both machine and human learners. %efficiency in teaching complex concepts, 
However, most algorithms are restricted to simple concept classes. 
In this work, we aim to understand the complexity of teaching the class of pLTL formulas. There has been extensive study in modeling a learner for temporal logic formulas, for example, see \cite{Hoxha2017,Kong2017TAC,Bombara2016,Neider,VazquezChanlatte2018LearningTS,zhe_info,zhe_ijcai2019,zheletter2,zhe_advisory,zhe2016,zhe_ADHS,NIPS2018Shah}, while it is much less understood in the context of machine teaching. In this paper, we abstract these learner models as preference-based version space learners, and focus on developing efficient algorithms for \emph{teaching} such learners.

\section{Parametric Linear Temporal Logic} 
In this section, we present an overview of parametric linear temporal logic  
(pLTL) \cite{pLTL2014,pLTL}. We start with the syntax and semantics of pLTL. The domain $\mathbb{B}=\{\top, \bot\}$ ($\top$ and $\bot$ represents True and False respectively) is
the Boolean domain and the time index set $\mathbb{T}=\{0, 1, \dots\}$ is a discrete set of natural numbers. We assume that there is an underlying system $\mathcal{H}$. The state $s$ of the
system $\mathcal{H}$ belongs to a finite set $S$ of states. A trajectory $\rho_L=s_0s_1\cdots s_{L-1}$ of length $L\in\mathbb{Z}_{>0}$ describing an evolution of the system $\mathcal{H}$ is a function from $\mathbb{T}$ to                    
$S$. A set $\mathcal{AP}=\{\pi_{1},\pi_{2},\dots,\pi_{n}\}$ is a set of atomic predicates. $\mathcal{L}: S\rightarrow2^{\mathcal{AP}}$ is a labeling function assigning a subset of atomic predicates in $\mathcal{AP}$ to each state $s\in S$.                                                                  
The syntax of the (F,G)-fragment bounded pLTL is defined recursively as\footnote{Although other temporal operators such as \textquotedblleft Until \textquotedblright ($\mathcal{U}$) may also appear in the full syntax of pLTL, they are omitted from the syntax of (F,G)-fragment bounded pLTL as they can be hard to interpret and are not often used for the inference of pLTL formulas.}
\[
\begin{split}
\phi:=&\top\mid\pi\mid\lnot\phi\mid\phi_{1}\wedge\phi_{2}\mid G_{\le \tau}\phi\mid F_{\le\tau}\phi,
\label{syntax}
\end{split}
\]
where $\pi$ is an \textit{atomic predicate}; $\lnot$ and $\wedge$ stand for negation and conjunction, respectively; $G_{\le \tau}$ and $F_{\le \tau}$ are temporal operators representing \textquotedblleft parameterized always\textquotedblright~and \textquotedblleft parameterized eventually\textquotedblright, respectively ($\tau\in\mathbb{T}$ is a temporal parameter). From the above-mentioned operators, we can also derive other operators such as $\vee$ (disjunction) and $\Rightarrow$ (implication). In the following content of the paper, we refer to (F,G)-fragment bounded pLTL as pLTL$_f$ for brevity.

%We can also derive $\vee$ (disjunction), $\Diamond$ (eventually), $\Box$ (always), $\mathcal{R}$ (release), $G_{\sim i}$ (parameterized always), $\mathcal{U}_{\sim i}$ (parameterized until), $\mathcal{R}_{\sim i}$ (parameterized release) and $\Rightarrow$ (implication) from the above-mentioned operators \cite{pLTL2014}. 

%The satisfaction of a pLTL formula $\phi$ as Boolean
%semantics can be found in \cite{pLTL2014}. We use $s_{1:L}\models\phi$ to denote the fact that
%a trajectory $s_{1:L}$ satisfies a pLTL formula $\phi$. If the satisfaction relations are evaluated at time index $k=1$, then we write $s_{1:L}\models_{\rm{S}}\phi$ for brevity.         
Next, we introduce the Boolean semantics of a pLTL$_f$ formula in the strong and the weak view \cite{Eisner2003,KupfermanVardi2001,Ho2014}. In the following, $(\rho_L,t)\models_{\rm{S}}\phi$ (resp. $(\rho_L,t)\models_{\rm{W}}\phi$)
means the trajectory $\rho_L$ strongly (resp. weakly) satisfies $\phi$ at time $t$, and $(\rho_L,t)\not\models_{\rm{S}}\phi$ (resp. $(\rho_L,t)\not\models_{\rm{W}}\phi$)
means the trajectory $\rho_L$ fails to strongly (resp. weakly) satisfy $\phi$ at time $t$.

\begin{definition}
	The Boolean semantics of the pLTL$_f$ in the strong view is defined recursively as
	\[
	\begin{split}
	(\rho_L,t)\models_{\rm{S}}\pi\quad\mbox{iff}\quad& t\le L-1~\mbox{and}~f(\rho_L(t))>0,\\
	(\rho_L,t)\models_{\rm{S}}\lnot\phi\quad\mbox{iff}\quad & (\rho_L,t)\not\models_{\rm{W}}\phi,\\
	(\rho_L,t)\models_{\rm{S}}\phi_{1}\wedge\phi_{2}\quad\mbox{iff}\quad &  (\rho_L,t)\models_{\rm{S}}\phi
	_{1}~\mbox{and}~(\rho_L,t)\models_{\rm{S}}\phi_{2},\\
	%		(\rho_L,t)\models_{\rm{S}}\phi_{1}\vee\phi_{2}\quad\mbox{iff}\quad &  (\rho_L,t)\models_{\rm{S}}\phi
	%		_{1}~\mbox{or}~(\rho_L,t)\models_{\rm{S}}\phi_{2},\\	
	(\rho_L,t)\models_{\rm{S}}F_{\le\tau}\phi\quad\mbox{iff}\quad &  \exists
	t^{\prime}\in[t,t+\tau], s.t.~(\rho_L,t^{\prime})\models_{\rm{S}}\phi,\\
	(\rho_L,t)\models_{\rm{S}}G_{\le\tau}\phi\quad\mbox{iff}\quad &  (\rho_L,t^{\prime})\models_{\rm{S}}\phi, \forall
	t^{\prime}\in[t, t+\tau]. 
	\end{split}
	\]
	\label{strong}
\end{definition}

\begin{definition}
	The Boolean semantics of the pLTL$_f$ in the weak view is defined recursively as
	\[
	\begin{split}
	(\rho_L,t)\models_{\rm{W}}\pi\quad\mbox{iff}\quad& \textrm{either}~t>L-1,\\
	& \mbox{or}~\big(t\le L-1~\mbox{and}~f(\rho_L(t))>0\big),\\
	(\rho_L,t)\models_{\rm{W}}\lnot\phi\quad\mbox{iff}\quad & (\rho_L,t)\not\models_{\rm{S}}\phi,\\	
	(\rho_L,t)\models_{\rm{W}}\phi_{1}\wedge\phi_{2}\quad\mbox{iff}\quad &  (\rho_L,t)\models_{\rm{W}}\phi
	_{1}~\mbox{and}~(\rho_L,t)\models_{\rm{W}}\phi_{2},\\	
	%		(\rho_L,t)\models_{\rm{W}}\phi_{1}\vee\phi_{2}\quad\mbox{iff}\quad &  (\rho_L,t)\models_{\rm{W}}\phi
	%		_{1}~\mbox{or}~(\rho_L,t)\models_{\rm{W}}\phi_{2},\\		
	(\rho_L,t)\models_{\rm{W}}F_{\le\tau}\phi\quad\mbox{iff}\quad & \exists 
	t^{\prime}\in[t,t+\tau], s.t.~(\rho_L,t^{\prime})\models_{\rm{W}}\phi,\\		
	(\rho_L,t)\models_{\rm{W}}G_{\le\tau}\phi\quad\mbox{iff}\quad & (\rho_L,t^{\prime})\models_{\rm{W}}\phi, \forall
	t^{\prime}\in[t, t+\tau].
	\label{weak}
	\end{split}
	\]
\end{definition}

\begin{figure}[th] 
	\centering
	\includegraphics[width=7cm]{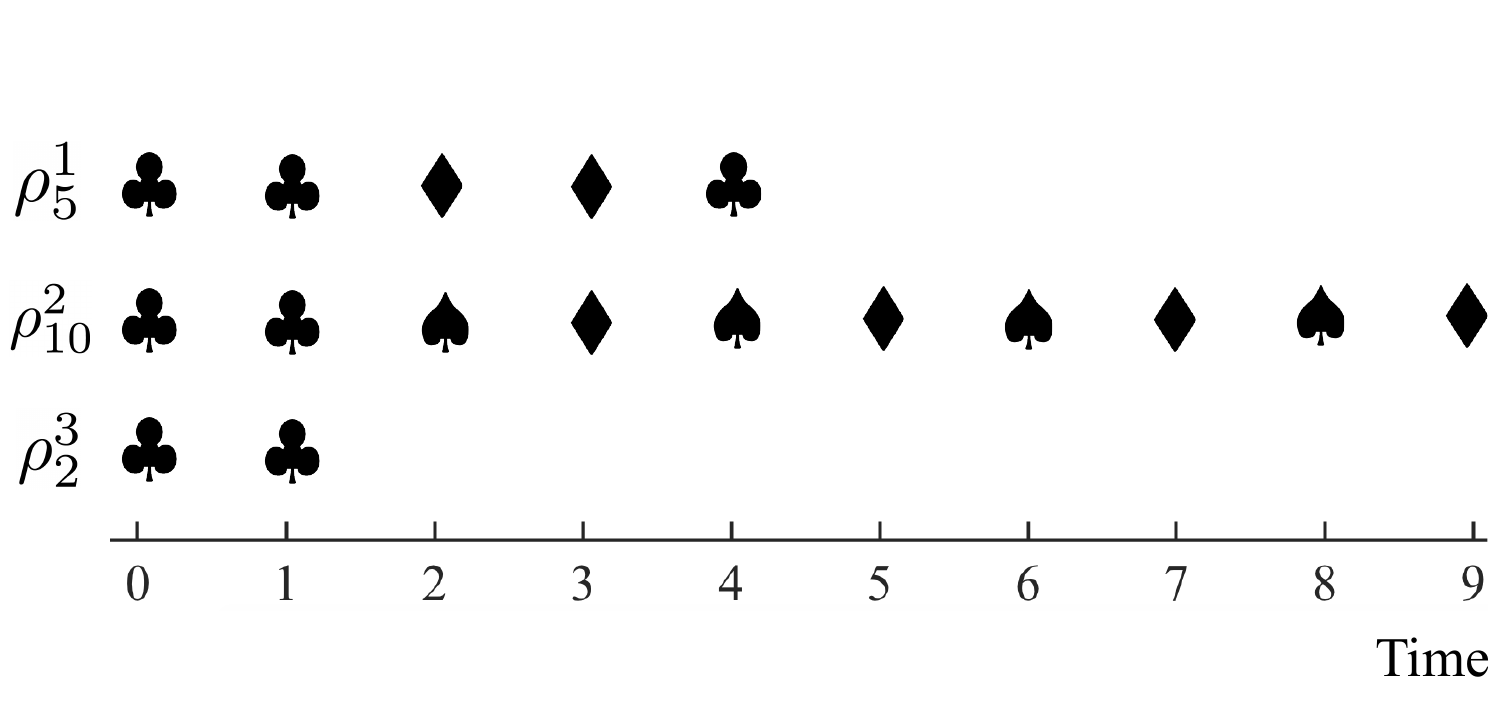}\caption{Three trajectories of different lengths.}
	\label{fig_demonstration}
\end{figure}

Intuitively, if a trajectory of finite length can be extended to infinite length, then the strong view indicates that the truth value of the formula on the infinite-length trajectory is already “determined” on the trajectory of finite length, while the weak view indicates that it may not be “determined” yet. As shown in the simple example in Fig. \ref{fig_demonstration}, with the set $S=\{\clubsuit, \spadesuit, \blacklozenge\}$ of states and three trajectories $\rho^1_5$, $\rho^2_{10}$ and $\rho^3_2$, $F_{\le 4}\clubsuit$ is strongly satisfied by all three trajectories, while $F_{\le 4}\spadesuit$ is strongly violated by $\rho^1_5$, strongly satisfied by $\rho^2_{10}$, and weakly satisfied (and also weakly violated) by $\rho^3_2$.

% From Definition \ref{strong} and Definition \ref{weak}, it can be seen that strong satisfaction or violation of a trajectory with respect to an pLTL$_f$ formula implies weak satisfaction or violation of the trajectory with respect to the same pLTL formula.       
   
\section{Teaching pLTL$_f$ Formulas with Demonstrations of Varying Lengths}        
\label{problem}
%In this section, 
We now provide the framework for teaching pLTL$_f$ formulas with demonstrations (trajectories) of varying lengths.

\subsection{Teaching Model} 
% We consider three different preference models: uniform preferences, (non-uniform) global preferences and local preferences.
Let $\mathcal{Z}:=S^{+}\times\{-1,1\}$ be the ground set of (labeled) demonstrations (trajectories), where $S^{+}$ denotes the \textit{Kleene plus} of $S$, and $\{-1,1\}$ is the set of labels. There is a \textit{hypothesis set} $\Phi=\{\phi_1, \dots, \phi_n\}$ consisting of $n\in\mathbb{Z}_{>0}$ hypothesis pLTL$_f$ formulas. The teacher knows a \textit{target hypothesis} $\phi^{\ast}\in\Phi$ and intends to teach $\phi^{\ast}$ to a learner.

% \begin{definition}         
We define the \textit{preference function} $\sigma: \Phi\times\Phi\rightarrow\mathbb{R}_{>0}$ as a function that encodes the learner's transition preferences. Specifically, given the current hypothesis $\phi$ and any two hypotheses $\phi'$ and $\phi''$, $\phi'$ is preferred to $\phi''$ from $\phi$ if and only if $\sigma(\phi';\phi)<\sigma(\phi'';\phi)$. If for any $\phi'$, $\sigma(\phi';\phi)$ does not depend on $\phi$, then $\sigma$ is a \textit{global} preference function; otherwise, $\sigma$ is a \textit{local} preference function. %For a global preference function $\sigma$, 
If $\sigma(\phi';\phi)$ is a constant for any $\phi$ and $\phi'$, then $\sigma$ is a \textit{uniform} preference function.

% For example, if the current hypothesis $\phi=F_{\le 3}\clubsuit$, and the learner prefers $\phi_1=F_{\le 5}\clubsuit$ to $\phi_2=F_{\le 5}\spadesuit$, then $\sigma(\phi_1;\phi)<\sigma(\phi_2;\phi)$. For global preference functions, the learner always prefers $\phi_1$ to $\phi_2$ at any current hypothesis, i.e., $\sigma(\phi_1; \cdot)<\sigma(\phi_2; \cdot)$. 

\begin{definition}                                                                                          
Given a target hypothesis $\phi^{\ast}$ and a demonstration $(\rho_{L}, l)$, where $l=1$ represents positive demonstration and $l=-1$ represents negative demonstration, $(\rho_{L}, l)$ is strongly inconsistent with a pLTL$_f$ formula $\phi\in\Phi$, if and only if the following condition is satisfied:
	\[
		\begin{split}
		\begin{cases}
		(\rho_{L},0)\models_{\rm{S}}\phi^{\ast}, ~~(\rho_{L},0)\models_{\rm{S}}\lnot\phi, ~~~~~~~~\mbox{if $l=1$};\\  
		(\rho_{L},0)\models_{\rm{S}}\lnot\phi^{\ast}, ~~(\rho_{L},0)\models_{\rm{S}}\phi, ~~~~~~~~\mbox{if $l=-1$}. 
		\end{cases}\\	 
		\end{split}
		\]
	\label{perfect}
	\vspace{-3mm}
\end{definition} 

\looseness -1 Starting from the teacher's first demonstration, the teaching stops as soon as the learner's current hypothesis reaches the target hypothesis, and we call it a \textit{teaching session}. For a sequence $D$ of demonstrations, we define the \textit{version space} induced by $D$, denoted as $\Phi(D)$, as the subset of pLTL$_f$ formulas in $\Phi$ that are \textit{not} strongly inconsistent with any demonstration in $D$. We use $\Theta$ to denote the random variable representing the randomness of the environment (e.g., learner's random choice of next hypothesis in the presence of ties) and $\theta$ to denote the \textit{realization} of $\Theta$ in each teaching session. %For example, with uniform preference function, the learner selects the next hypothesis randomly in the version space.  

\begin{definition}  
 We define a teaching setting as a 4-tuple 
 $\mathcal{S}=(\phi^{\ast}, \phi^0, \Phi, \dom(\Theta))$, where $\phi^{\ast}$ is the target hypothesis, $\phi^0\neq\phi^{\ast}$ is the learner's initial hypothesis, $\Phi$ is the hypothesis set and $\dom(\Theta)$ is the domain of the random variable $\Theta$. 
\end{definition}  

Let $\mathcal{T}$ %: \Phi\times2^{\Phi}\rightarrow2^\mathcal{Z}$ 
be a teacher that outputs a sequence of demonstrations based on the target hypothesis, the version space, and either the learner's initial hypothesis (non-adaptive teacher) or the learner's current hypothesis (adaptive teacher). %maps the learner's current hypothesis and current version space to a sequence of labeled demonstrations. 
For a learner with preference function $\sigma$, let $\mathcal{L}_{\sigma}: \Phi\times2^{\Phi}\times\mathcal{Z}\times\dom(\Theta)\rightarrow\Phi$ be the learner's mapping that maps the learner's current hypothesis, the current version space, the current demonstration and the randomness of the environment to the learner's next updated hypothesis. Given a teacher $\mathcal{T}$, a learner $\mathcal{L}_{\sigma}$ and the randomness $\theta$ in the teaching setting $\mathcal{S}=(\phi^{\ast}, \phi^0, \Phi, \dom(\Theta))$, we use $D_{\mathcal{S}}(\mathcal{T}, \mathcal{L}_{\sigma}, \theta)$ to denote the sequence of demonstrations provided by the teacher $\mathcal{T}$ before the learner's hypothesis reaches $\phi^*$. 

%Given the learner's initial hypothesis $\phi^0$ and the hypothesis set $\Phi$, we use $\mathcal{T}^1(\phi^0, \Phi), \dots, \mathcal{T}^K(\phi^0, \Phi)$ to denote the sequence of demonstrations provided by the teacher $\mathcal{T}$.
%For any $k\in[1,K]$ and $\theta\in \dom{\Theta}$, we use $\phi^{k}:=\mathcal{L}_{\sigma}(\phi^{k-1}, \Phi^{k-1},\mathcal{T}^{k}(\phi^0, \Phi), \theta)$ to denote the learner's hypothesis after $k$ demonstrations.   

% , and $\Phi^k :=\Phi(\{\mathcal{T}^1(\phi^0, \Phi), \dots, \mathcal{T}^k(\phi^0, \Phi)\})$ to denote the version space after $k$ demonstrations

	\begin{definition}
	\looseness -1 
	%:=\big(\mathcal{T}^1(\phi^0, \Phi), \dots, \mathcal{T}^K(\phi^0, \Phi)\big)
	%We define t
	The \textit{accumulated-number (AN) teaching cost} and \textit{accumulated-length (AL) teaching cost} are defined as %respectively as follows.
		\[
	 	\begin{split} 
		\textrm{AN-Cost}_{\mathcal{S}}(\mathcal{T}, \mathcal{L}_{\sigma}, \theta):=&\vert D_{\mathcal{S}}(\mathcal{T}, \mathcal{L}_{\sigma}, \theta)\vert,\\
		\textrm{AL-Cost}_{\mathcal{S}}(\mathcal{T}, \mathcal{L}_{\sigma}, \theta):=&\sum\limits_{(\rho^k_{L_k},  l_k)\in D_{\mathcal{S}}(\mathcal{T}, \mathcal{L}_{\sigma}, \theta)}L_k.
		\end{split}
		\]
		%where $D_{\mathcal{S}}(\mathcal{T}, \mathcal{L}_{\sigma}, \theta):=\big(\mathcal{T}^1(\phi^0, \Phi), \dots, \mathcal{T}^K(\phi^0, \Phi)\big)$ is a sequence of demonstrations,
	    %$\phi^k\neq\phi^{\ast}$ for any $k\in[0,K-1]$, and $\phi^K=\phi^{\ast}$.
		
% 		\[
% 		\begin{split}
		
% 		\end{split}
% 		\]		
		\label{cost}
	\end{definition} 
	
% $D=\{(\rho^1_{L_1}, l_1), (\rho^2_{L_2}, l_2), \dots\}$ is a set of labeled demonstrations. We also say $\phi^{\ast}$ is teachable from $D$ if and only if $\Phi(D)=\{\phi^{\ast}\}$.
% % 		\begin{align}
% 		\begin{split}
		
% 		\end{split}
% 		\label{teachingSet}
% 		\end{align}	

Intuitively, the AN teaching cost and the AL teaching cost are the number of demonstrations and the accumulated time lengths of the demonstrations for the learner's hypothesis to reach $\phi^{\ast}$ in a specific teaching session, respectively. 
\begin{definition}                  
Given a teacher $\mathcal{T}$ and a learner $\mathcal{L}_{\sigma}$ in the teaching setting $\mathcal{S}=(\phi^{\ast}, \phi^0, \Phi, \dom(\Theta))$, we define the worst-case AN teaching cost and AL teaching cost as%\vspace{-3mm}
	    \[
		\begin{split}
		\textrm{AN-Cost}^{\textrm{WC}}_{\mathcal{S}}(\mathcal{T}, \mathcal{L}_{\sigma}):=&\max\limits_{\theta\in \dom(\Theta)}\textrm{AN-Cost}_{\mathcal{S}}(\mathcal{T}, \mathcal{L}_{\sigma}, \theta),\\
		\textrm{AL-Cost}^{\textrm{WC}}_{\mathcal{S}}(\mathcal{T}, \mathcal{L}_{\sigma}):=&\max\limits_{\theta\in \dom(\Theta)}\textrm{AL-Cost}_{\mathcal{S}}(\mathcal{T}, \mathcal{L}_{\sigma}. \theta).
		\end{split}
		\]
  \label{WCTC}
  \vspace{-3mm}
\end{definition} 

\begin{definition}              
Given a learner $\mathcal{L}_{\sigma}$ in the teaching setting $\mathcal{S}=(\phi^{\ast}, \phi^0, \Phi, \dom(\Theta))$, we define the \textit{AN teaching complexity} and \textit{AL teaching complexity} respectively as 
		\[
		\begin{split}
		\textrm{AN-Complexity}_{\mathcal{S}}(\mathcal{L}_{\sigma}):=&\min\limits_{\mathcal{T}}\textrm{AN-Cost}^{\textrm{WC}}_{\mathcal{S}}(\mathcal{T}, \mathcal{L}_{\sigma}),\\
		\textrm{AL-Complexity}_{\mathcal{S}}(\mathcal{L}_{\sigma}):=&\min\limits_{\mathcal{T}}\textrm{AL-Cost}^{\textrm{WC}}_{\mathcal{S}}(\mathcal{T}, \mathcal{L}_{\sigma}).
		\end{split}
		\]
		\label{complexity}
		\vspace{-3mm}
	\end{definition} 
Intuitively, the AN teaching complexity and the AL teaching complexity are the minimal number of demonstrations and the minimal accumulated time lengths of the demonstrations needed for the learner's hypothesis to reach $\phi^{\ast}$ despite the randomness of the environment, respectively. 

For example, we consider the hypothesis set $\Phi=\{F_{\le i}s\}$, where $i\in\{0,\dots,4\}$, and $s\in S=\{\clubsuit, \spadesuit, \blacklozenge\}$. If $\phi^{\ast}$ is $F_{\le 2}\clubsuit$, and the learner has uniform preference, then for any $\phi^0$, one sequence of demonstrations which can minimize both the (worst-case) 
AN and AL teaching costs\footnote{Note that there does not always exist a sequence of demonstrations that minimizes both the AN and AL teaching costs.} are firstly a negative demonstration of $\spadesuit, \blacklozenge, \spadesuit, \clubsuit, \clubsuit$, and then a positive demonstration of $\spadesuit, \blacklozenge,\clubsuit$. In this teaching setting, the AN and AL teaching complexities are 2 and 8, respectively.

\subsection{TLIP: Teaching pLTL$_f$ Formulas with Integer Programming} 
\label{sec_TLIP}
Finding the optimal sequence of demonstrations with minimal AN or AL teaching cost has time complexity in the order of $2^{\vert D_{L_{\textrm{max}}}\vert}$, where $D_{L_{\textrm{max}}}$ is the set of all possible demonstrations with length at most $L_{\textrm{max}}$. As $\vert D_{L_{\textrm{max}}}\vert=\sum_{i=1}^{L_{\textrm{max}}}\vert S\vert^{L_{\textrm{max}}}$, $2^{\vert D_{L_{\textrm{max}}}\vert}$ is doubly exponential with the maximal length of the demonstrations.

% Specifically, we compute each labeled demonstration $(\rho_L, l)$ such that the ratio of the number of hypothesis pLTL$_f$ formulas in the version space which are strongly inconsistent with $(\rho_L, l)$ to the length $L$ is maximal.

%To reduce computation, w
We resort to greedy methods for \textit{myopic teaching} with near-optimal performance. To compute the greedy solution, we first derive a necessary condition through Definition \ref{minimal_length} and Theorem \ref{gamma0} for the minimal time length of a demonstration so that a set of pLTL$_f$ formulas are strongly inconsistent with (thus can be eliminated by) this demonstration.
	
\begin{definition}   
We define the \textit{minimal time length} $\zeta(\phi,l)$ of a pLTL$_f$ formula $\phi$ with respect to a label $l$ recursively as 
		\[ 
		\begin{split}
		\zeta(\pi,l)=& 0,~~~~~~
		\zeta(\lnot\phi, l) = \zeta(\phi,-l),\\
		\zeta(\phi_1\wedge\phi_2, l) =& 
		\begin{cases}
		\max\{\zeta(\phi_1, l), \zeta(\phi_2, l)\},& \mbox{if $l=1$};\\  
		\min\{\zeta(\phi_1, l), \zeta(\phi_2, l)\},& \mbox{if $l=-1$},
		\end{cases}\\
		\zeta(F_{\le\tau}\phi, l) =&
		\begin{cases}
		\zeta(\phi, l),~~~~~~~~~~~~~~~~~~~~~~\mbox{if $l=1$};\\  
		\zeta(\phi, l)+\tau,~~~~~~~~~~~~~~~\mbox{if $l=-1$},
		\end{cases}\\	
		\zeta(G_{\le \tau}\phi, l) =&
		\begin{cases}
		\zeta(\phi, l)+\tau,~~~~~~~~~~~~~~~~\mbox{if $l=1$};\\  
		\zeta(\phi, l),~~~~~~~~~~~~~~~~~~~~~~~\mbox{if $l=-1$}.                                                                                                                   
		\end{cases}
		\end{split}                                   
		\]
		\label{minimal_length}
\end{definition} 

	\begin{theorem}
		Given a target hypothesis $\phi^{\ast}$ and the hypothesis set $\Phi$, if a demonstration $(\rho_{L}, l)$ is strongly inconsistent with a subset $\hat{\Phi}=\{\phi_{i}\}^{\hat{N}}_{i=1}\subset\Phi$ of pLTL$_f$ formulas, then
		
		\[
		L\ge \max\{\zeta(\phi^{\ast}, l), \max\limits_{1\le i\le \hat{N}}\zeta(\phi_i, -l)\}.
		\]
		\label{gamma0}                 	    	 
	\end{theorem}

% with $c(\phi_j, \rho)=-c(\phi^{\ast}, \rho)\neq0$ 
%\setlength{\intextsep}{0pt} 
	\begin{algorithm}[!ht]
		\DontPrintSemicolon
		\SetKwBlock{Begin}{function}{end function}    
		{
			$\textbf{Input}$: hypothesis set $\Phi$, initial hypothesis $\phi^0$
		}  \\
		{
		Initialize $k\gets0$, $\tilde{\Phi}\gets\emptyset$, $\Phi^0\gets\Phi$
		} \\    
	\While{$\phi^k\neq\phi^{\ast}$}{
		\lIf{$\text{MyopicTeacher}=1$}
		{
		$\phi^{k\ast}=\phi^{\ast}$ 
		}     
		\lElse
	    {
	    Compute $\phi^{k\ast}\gets \text{Oracle}(\phi^k, \Phi^k, \phi^{\ast})$ 
	    }
	    \If{$\sigma$ is global}
		{
		 $\tilde{\Phi}\gets\{\phi\in\Phi^k:\sigma(\phi;\cdot)\le\sigma(\phi^{k\ast}; \cdot)\}$ \label{line_global}
		}
		\Else
	    {
	    $\tilde{\Phi}\gets\{\phi\in\Phi^k:\sigma(\phi;\phi')\le\sigma(\phi^{k\ast}; \phi')$ for some $\phi'\in\Phi^k\}$ \label{line_local}
	    }
		{
	    $(\rho^k, \ell^k), \hat{\Phi}\gets$ComputeDemonstration($\Phi^k$, $\tilde{\Phi}$, $\phi^{k\ast}$) \label{line_compute}
	    }\\
	{$\Phi^{k+1}\gets\Phi^k\setminus\hat{\Phi}$, $\tilde{\Phi}\gets\tilde{\Phi}\setminus\hat{\Phi}$, $k\gets k+1$ \label{line_eliminate}} \\ 
	    \If{$\text{AdaptiveTeacher}=0$}
		{
	        \lIf{$\tilde{\Phi}=\{\phi^{\ast}\}$}
	        {$\phi^k\gets\phi^{\ast}$}
	        \lElse{
		   Randomly select $\phi^k\neq\phi^{\ast}$ from $\tilde{\Phi}$}
             \label{line_select}
		} \lElse
		{ 
	Observe the learner's next hypothesis $\phi^{k}$
		}
		}
		{$K\gets k-1$}\\
	\Return{$\{(\rho^0, \ell^0), (\rho^1, \ell^1), \dots, (\rho^K, \ell^K)\}$} 
		\caption{Teaching of pLTL$_f$ Formulas with Integer Programming (TLIP)}                                          
		\label{compute_TLIP}
		%\vspace{-3mm}
	\end{algorithm}	
	%\setlength{\textfloatsep}{0pt}% Remove \textfloatsep
	
% We first derive a necessary condition for the minimal time length of demonstrations for eliminating pLTL formulas from the version space. 

    \begin{algorithm}
		\DontPrintSemicolon
		\SetKwBlock{Begin}{function}{end function}     
		{
			$\textbf{Input}: \Phi, \tilde{\Phi}$, $\phi^{\ast}$
		}  \\
		{
			Compute $\rho^{\ast}_{\textrm{pos}}$ and $\kappa(\rho^{\ast}_{\textrm{pos}})$ for $\textrm{IP}_{\textrm{pos}}(\tilde{\Phi}, \phi^{\ast})$ 
		} \\                         
		{
			Compute $\rho^{\ast}_{\textrm{neg}}$ and $\kappa(\rho^{\ast}_{\textrm{neg}})$ for $\textrm{IP}_{\textrm{neg}}(\tilde{\Phi}, \phi^{\ast})$

		}                       
		\If{$\kappa(\rho^{\ast}_{\textrm{pos}})\ge\kappa(\rho^{\ast}_{\textrm{neg}})$}
		{
			$(\rho, \ell)\gets (\rho^{\ast}_{\textrm{pos}}, 1)$,
			$\hat{\Phi}\gets\{\phi\in\Phi:c(\phi,\rho)=-1\}$
		} 
		\Else
		{
			$(\rho, \ell)\gets (\rho^{\ast}_{\textrm{neg}}, -1)$,
			$\hat{\Phi}\gets\{\phi\in\Phi:c(\phi,\rho)=1\}$
		}                          
		\Return{$(\rho, \ell)$, $\hat{\Phi}$} 
		\caption{ComputeDemonstration}                                           
		\label{compute}
	\end{algorithm}	

\looseness -1 \algref{compute_TLIP} shows the proposed TLIP approach for teaching pLTL$_f$ formulas to learners with preferences. Here we focus on the myopic solution ($\textit{MyopicTeacher}=1$) under the \textit{non-adaptive} setting ($\textit{AdaptiveTeacher}=0$) where the teacher does not observe the learner's current hypothesis and provides the sequence of demonstrations based on the learner's initial hypothesis. We compute $\tilde{\Phi}$ as either (i) the set of hypotheses that are preferred over the target hypothesis in the current version space if the learner has global preferences (Line \ref{line_global}), or (ii) the union of the sets of hypotheses that are preferred over target hypothesis based on each hypothesis in the current version space if the learner has local preferences (Line \ref{line_local}). Then we call \algref{compute} to compute the demonstrations that achieve the greedy myopic solution (Line \ref{line_compute}). 

% Initially, index $k$ is 0, $\Phi^k$ is set as the hypothesis set $\Phi$, $\phi^k$ is set as $\phi^0$, and $\mathcal{D}$ is set as the empty set.

Note that finding the greedy myopic solution 
via exhaustive search amounts to traversing the space of demonstrations, which is exponential with the maximal length of the demonstrations. We propose to find the greedy solution via a novel integer programming (IP) formulation. For a trajectory $\rho_L$ and a pLTL$_f$ formula $\phi$, we denote $c(\phi, \rho_L)=1$ if $(\rho_L,0)\models_{\rm{S}} \phi$; $c(\phi, \rho_L)=-1$ if $(\rho_L,0)\models_{\rm{S}} \lnot\phi$; and $c(\phi, \rho_L)=0$ if $(\rho_L,0)\not\models_{\rm{S}} \phi$ and $(\rho_L,0)\not\models_{\rm{S}} \lnot\phi$. For positive demonstrations, we compute the following integer programming problem 
$\textrm{IP}_{\textrm{pos}}(\tilde{\Phi}, \phi^{\ast})$.
\[      
\begin{split}
\max_{\rho_L} ~ & \kappa(\rho_L)                          \\
\text{subject to:} ~ & b_j\in \{0,1\}, \forall j, ~\mbox{s.t.}~ \phi_j\in\tilde{\Phi},~~ c(\phi^{\ast},\rho_L)= 1, \\
& c(\phi_j,\rho_L)= 1-2b_j, \forall j, ~\mbox{s.t.}~ \phi_j\in\tilde{\Phi}, \\
& L\ge\zeta(\phi^{\ast}, 1), L\ge b_j \zeta(\phi_j, -1), \forall j, ~\mbox{s.t.}~ \phi_j\in\tilde{\Phi},
\end{split}
\]
where $\kappa(\rho_L)=\big(\sum_{\phi_j\in\tilde{\Phi}}b_j\big)$ when optimizing for the AN teaching cost and $\kappa(\rho_L)=\big(\sum_{\phi_j\in\tilde{\Phi}}b_j\big) /L$ when optimizing for the AL teaching cost, the strong satisfaction or strong violation of a pLTL$_f$ formula $\phi$ by $\rho_L$ can be encoded as integer linear constraints of $\rho_L$, and the constraints for $L$ are obtained from Theorem \ref{gamma0}.  
In practice, the problem $\textrm{IP}_{\textrm{pos}}(\tilde{\Phi}, \phi^{\ast})$ can be efficiently solved by highly-optimized IP solvers \cite{gurobi}, %(cf \secref{implementation} for details)
which, as demonstrated in \secref{implementation}, is significantly more efficient than the exhaustive search method.  

For negative demonstrations, the integer programming problem $\textrm{IP}_{\textrm{neg}}(\tilde{\Phi}, \phi^{\ast})$ can be similarly formulated with the constraints $c(\phi^{\ast},\rho_L)= -1$ and $c(\phi_j,\rho_L)= -1+2b_j, \forall j, ~\mbox{s.t.}~ \phi_j\in\tilde{\Phi}$. We use $\rho^{\ast}_{\textrm{pos}}$ and $\rho^{\ast}_{\textrm{neg}}$ to denote the optimal positive and negative demonstration computed from $\textrm{IP}_{\textrm{pos}}(\tilde{\Phi}, \phi^{\ast})$ and $\textrm{IP}_{\textrm{neg}}(\tilde{\Phi}, \phi^{\ast})$, respectively. We select $\rho^{\ast}_{\textrm{pos}}$ or $\rho^{\ast}_{\textrm{neg}}$ depending on whether $\kappa(\rho^{\ast}_{\textrm{pos}})$ is no less than $\kappa(\rho^{\ast}_{\textrm{neg}})$ or not. Then, we eliminate the hypothesis pLTL$_f$ formulas that are strongly inconsistent with the selected demonstration (Line \ref{line_eliminate}). For non-adaptive teaching, we randomly select a pLTL$_f$ formula different from the target hypothesis (Line \ref{line_select}, as we consider the worst case) and perform another round of computation for the demonstration until the current hypothesis reaches the target hypothesis.

\subsection{Teaching with Positive Demonstrations Only} 	
Learning temporal logic formulas from positive demonstrations is a typical problem in temporal logic inference \cite{zhe_info,VazquezChanlatte2018LearningTS,NIPS2018Shah}. 

The algorithm for teaching pLTL$_f$ formulas with positive demonstrations only can be modified from Algorithms 1 and 2, by deleting Lines 3-7 of Algorithm 2 and obtaining $(\rho, l)$ as $(\rho^{\ast}_{\textrm{pos}},1)$. The following theorem provides a necessary condition for successfully teaching a pLTL$_f$ formula with positive demonstrations to a learner with global preferences. 
\begin{theorem} 
	Given a hypothesis set $\Phi$ and a sequence of positive demonstrations $D_{\textrm{p}}$, if a target hypothesis $\phi^{\ast}\in\Phi$ is \textit{teachable} from $D_{\textrm{p}}$ to a learner with global preference function $\sigma$, i.e., $\forall \phi\in\Phi(D)\setminus\{\phi^{\ast}\}, \sigma(\phi; \cdot)>\sigma(\phi^{\ast}; \cdot)$, then it holds that $\max\limits_{1\le k\le \vert D_{\textrm{p}}\vert}L_k\ge \max\limits_{\phi_i\in\tilde{\Phi}\setminus\{\phi^{\ast}\}}\zeta(\phi_i, -1)$, and that there does not exist a pLTL$_f$ formula $\phi'$ with $\sigma(\phi';\cdot)\le\sigma(\phi^{\ast};\cdot)$ such that $\phi^{\ast}\Rightarrow\phi'$. Here, $\tilde{\Phi}:=\{\phi\in\Phi:\sigma(\phi;\cdot)\le\sigma(\phi^{\ast}; \cdot)\}$, $L_k$ is the time length of the $k$-th demonstration in $D_{\textrm{p}}$, and $\zeta(\phi_i, -1)$ is as defined in Definition \ref{minimal_length}.
	\label{prop_pos}                 	    	            
\end{theorem}  
% \begin{theorem} 
% 	Given a hypothesis set $\Phi$ and a sequence of positive demonstrations $D_{\textrm{p}}=\{\rho^1_{L_1}, \dots, \rho^{N}_{L_{N}}\}$, if a teacher can always teach a target hypothesis  $\phi^{\ast}\in\Phi$ from $D_{\textrm{p}}$ to learners with global preference function $\sigma$, then there does not exist a pLTL$_f$ formula $\phi$ with $\sigma(\cdot,\phi)<\sigma(\cdot,\phi^{\ast})$ such that $\phi^{\ast}\Rightarrow\phi$, and 
% 	$\max\limits_{1\le k\le N}L_k\ge \max\limits_{\phi_i\in\tilde{\Phi}\setminus\{\phi^{\ast}\}}\zeta(\phi_i, -1)$, where $\tilde{\Phi}:=\{\phi\in\Phi:\sigma(\phi;\cdot)\le\sigma(\phi^{\ast}; \cdot)\}$, $\zeta(\phi_i, -1)$ is as defined in Def. \ref{minimal_length}.
% 	\label{prop_pos}                 	    	            
% \end{theorem}  

\section{Adaptive Teaching of pLTL$_f$ Formulas}                                   
\label{sec_adaptive}
We now explore the theoretical aspects of machine teaching for pLTL$_f$ formulas under the adaptive setting. %when the teacher has access to the learner's current hypothesis (adaptivity), when the teacher provides intermediate target hypotheses (oracles) and when the teacher teaches with positive demonstrations only.

%In this section, we explore teaching of temporal logic formulas when the teacher has access to the learner's current hypothesis (adaptivity), when the teacher provides intermediate target hypotheses (oracles) and when the teacher teaches with positive demonstrations only.

% \subsection{Preference-based teaching} 

% \begin{definition}
% 	Given a target hypothesis $\phi^{\ast}$ and the hypothesis space $\Phi$ with global preferences, we define the \textit{AN preference-based teaching cost} and \textit{AL preference-based teaching cost} respectively as follows.\\
% 	$\textrm{PBTC}_{\textrm{N}}(\phi^{\ast}, \Phi):=\min\limits_{D}\vert D\vert$, s.t. $\forall \phi\in\Phi(D)\setminus\{\phi^{\ast}\}, \sigma(\phi; \cdot)>\sigma(\phi^{\ast}; \cdot)$, and\\
% 	$\textrm{PBTC}_{\textrm{L}}(\phi^{\ast}, \Phi):=\min\limits_{D}\displaystyle\sum_{1\le i\le N_D}L_i$, s.t. $\forall \phi\in\Phi(D)\setminus\{\phi^{\ast}\}, \sigma(\phi; \cdot)>\sigma(\phi^{\ast}; \cdot)$.
% \end{definition}

\subsection{Teaching Complexity} 
Different from non-adaptive teaching, an adaptive teacher observes the learner's current hypothesis and provides the next demonstration according to the target hypothesis, the current version space and the learner's current hypothesis. 
% While adaptive teaching cannot improve the performance for uniform and global preferences \cite{chen18adaptive}, adaptive teaching has an advantage for learners with local preferences in the presence of uncertainties, as it can adjust its demonstrations after observing the current hypothesis. 
%
% For example, if the current hypothesis $\phi=F_{\le 4}\clubsuit$, and the learner may have equal preference for $\phi_1=F_{\le 5}\clubsuit$ to $\phi_2=F_{\le 5}\spadesuit$, then the learner will randomly select $\phi_1$ or $\phi_2$.
%
% The AN and AL teaching costs and teaching complexities in the adaptive setting can be defined as in Def. \ref{cost} and \ref{complexity} except that $D_{\mathcal{S}}(\mathcal{T},\mathcal{L}_{\sigma}, \theta):=\big(\mathcal{T}^1(\phi^0, \Phi^0), \dots, \mathcal{T}^1(\phi^K, \Phi^K)\big)$,
% where $\Phi^0=\Phi$, $\phi^k\neq\phi^{\ast}, \forall k\in[0,K-1]$ and $\phi^K=\phi^{\ast}$.
\algref{compute_TLIP} with $\textit{AdaptiveTeacher}=1$ shows the procedure for adaptive teaching using TLIP. 

% Before each demonstration, the adaptive teaching algorithm computes the demonstration based on the current hypothesis of the learner.

The following theorem provides near-optimality guarantees under the adaptive myopic setting \cite{chen18adaptive}.

% \footnote{In the case of ties, we assume that the myopic teacher prefers demonstrations that make learner stay at the same hypothesis.} 

\begin{theorem}
	We denote the myopic adaptive teacher in TLIP as $\mathcal{T}^{\textrm{m}}$. Given a target hypothesis $\phi^{\ast}$ and the hypothesis set $\Phi$,
	then
	\[
    \begin{split}
    \textrm{AN-Cost}^{\textrm{WC}}_{\mathcal{S}}(\mathcal{T}^{\textrm{m}}, \mathcal{L}_{\sigma})\le\lambda(\log \vert\tilde{\Phi}_{\mathcal{S}}\vert+1)\textrm{AN-Complexity}_{\mathcal{S}}(\mathcal{L}_{\sigma}),\\
    \textrm{AL-Cost}^{\textrm{WC}}_{\mathcal{S}}(\mathcal{T}^{\textrm{m}}, \mathcal{L}_{\sigma})\le\lambda(\log \vert\tilde{\Phi}_{\mathcal{S}}\vert+1)\textrm{AL-Complexity}_{\mathcal{S}}(\mathcal{L}_{\sigma}),
    \end{split}
    \]
	where $\tilde{\Phi}_{\mathcal{S}}:=\{\phi\in\Phi:\sigma(\phi;\phi^0)\le\sigma(\phi^{\ast}; \phi^0)\}$, $\lambda=1$ if $\sigma$ is global, and $\lambda=2$ if $\sigma$ is local and the following two conditions are satisfied for both $\sigma$ and the sequence $D$ of demonstrations.  
	\[                                                                  
	\begin{split}
	\begin{cases}
	1. \forall \phi', \phi''\in \Phi, \sigma(\phi'; \phi)\le\sigma(\phi''; \phi)\le\sigma(\phi^{\ast}; \phi)\\
	~~~~~~~~~~~~~~~~~~\Rightarrow \sigma(\phi''; \phi')\le\sigma(\phi^{\ast}; \phi');     \\  
	2. \forall \Phi'\subset \bar{\Phi}(\{(\rho,l)\}), \exists (\rho',l')\in D, s.t., \bar{\Phi}(\{(\rho',l')\})=\Phi'. 
	\end{cases}\\	 
	\end{split}
	\]
	In Condition 2, $\bar{\Phi}(\{(\rho,l)\})$ denotes the set of hypotheses in $\Phi$ which are strongly inconsistent with demonstration $(\rho,l)$.
	\label{th_bound}     
	
\end{theorem}

% :=\{\phi'\in\Phi: c(\phi^{\ast},\rho)=-c(\phi^{\ast},\rho)\neq0\}

% 	$\textrm{TC}^{m}_{\textrm{N}}(\phi^{\ast}, \Phi)$ and $\textrm{TC}^{m}_{\textrm{L}}(\phi^{\ast}, \Phi)$ are the AN teaching cost and AL teaching cost of the myopic teaching algorithm, respectively.
	
% As uniform preference and global preference satisfy conditions 1 and 2, Theorem \ref{complexity} shows that the myopic teaching algorithm can achieve near-optimal performance within logarithm factors. For local preference functions, conditions 1 and 2 are not necessarily satisfied, thus there is no theoretical guarantee for near-optimal performance. 
Theorem \ref{th_bound} shows that, under the adaptive myopic setting, TLIP can achieve near-optimal performance for global preferences and certain local preferences that satisfy Conditions 1 and 2. This motivates us to design intermediate target hypotheses as shown in \secref{sec_oracle}.
  
\subsection{Teaching with Oracles} 	
\label{sec_oracle}
For learners with local preferences, we can design intermediate target hypotheses so that Condition 1 of Theorem \ref{th_bound} can be satisfied for learning each target hypothesis. %We call such intermediate target hypotheses \textit{oracles}. 
In \algref{compute_TLIP}, we assume that we have access to an \emph{oracle}, i.e., $\text{Oracle}(\phi^k, \Phi^k, \phi^{\ast})$, which outputs an intermediate target hypothesis $\phi^{k\ast}$ at each step. 

% We refer to such scenario as teaching with oracles.

For example, if the learner's current hypothesis $\phi^{\textrm{c}}$ does not contain $\clubsuit$, then the learner prefers formulas with the same temporal operator as that in $\phi^{\textrm{c}}$; and if $\phi^{\textrm{c}}$ contains $\clubsuit$, then the learner prefers G-formulas than F-formulas. With the same temporal operator, the learner prefers formulas with $\spadesuit$ than formulas with $\clubsuit$, and prefers formulas with $\blacklozenge$ the least. Then, for $\phi^{\ast}=G_{\le 2}\clubsuit$, $\phi^1=F_{\le 3}\spadesuit$, $\phi^2=F_{\le 5}\clubsuit$, $\phi^3=F_{\le 3}\blacklozenge$,
we have $\sigma(\phi^2;\phi^1)<\sigma(\phi^3;\phi^1)<\sigma(\phi^{\ast};\phi^1)$, but $\sigma(\phi^{\ast};\phi^2)<\sigma(\phi^3;\phi^2)$. Therefore, Condition 1 of Theorem \ref{th_bound} does not hold. If the oracle outputs $\phi^2$ as an intermediate target hypothesis before the learner's hypothesis reaches $\phi^2$, then the teaching problem is decomposed into two subproblems, i.e., teaching before $\phi^2$ and after $\phi^2$, and both subproblems satisfy Condition 1 of Theorem \ref{th_bound}.

\section{Experiments}                                                           
\label{implementation}
In this section, we evaluate the proposed approach under different teaching settings.
%Due to space limitations, another case study in a robotic navigation scenario can be found at Dropbox link \url{http://bit.ly/2R7LMPj}.}. 
The hypothesis set of pLTL$_f$ formulas are listed in Table \ref{table_space}. The set $S$ of states is $\{0, 1, \dots, 10\}$. In each teaching session, we randomly select both the initial hypothesis and the target hypothesis from the hypothesis set. All results are averaged over 10 teaching sessions.
% We set the ground-truth formula as ${\phi}^{\ast}=F_{\le3}(x\le5)$. 
  \begin{table}[h]
  	\begin{center}
  	
  		\caption{Hypothesis set of pLTL$_f$ formulas ($a=5, 10, 15$ correspond to hypothesis sets of sizes 90, 180 and 270, respectively). } \label{table_space}
%  		\rowcolors{2}{gray!25}{white}
        \scalebox{1}{
  		\begin{tabular}{|c|p{0.5\textwidth}|} 
  			\hline
  			Category & Hypothesis pLTL$_f$ Formulas \\
  			\hline
  			F-formulas &$F_{\le1}(x\le1)$, \dots,$F_{\le1}(x\le9)$, \newline %$F_{\le2}(x\le1)$, \dots,$F_{\le2}(x\le9)$\newline 
  			\dots \newline 
  			$F_{\le a}(x\le1)$, \dots,$F_{\le a}(x\le9)$\\
  			\hline
  			G-formulas &$G_{\le1}(x\le1)$, \dots, $G_{\le1}(x\le9)$, \newline %$G_{\le2}(x\le1)$, \dots, $G_{\le2}(x\le9)$ \newline 
  			\dots \newline $G_{\le a}(x\le1)$, \dots, $G_{\le a}(x\le9)$\\
  			\hline
  		\end{tabular}
  		}
  	\end{center}
  \end{table}
% We consider the following three specific types of preferences:\\
% %\textbf{Uniform preferences}: the learner has equal preferences for any pLTL$_f$ formula in the hypothesis set.\\
% \textbf{Global preferences}: the learner prefers F-formulas than G-formulas, and with the same temporal operator the learner prefers smaller temporal parameter $i$ (as in $F_{\le i}(x\le 5)$). \\

%\vspace{0.2in}
\subsection{Teaching pLTL$_f$ Formulas under Global Preferences}\label{sec:exp:global}      
We first compare TLIP with the exhaustive search method for myopic teaching (ESMT).
Table \ref{computation_time} shows the computation time for TLIP using myopic teaching and ESMT (minimizing the AL teaching costs) for learners with uniform preferences, where timeout (TO) is 300 minutes. 
ESMT becomes intractable when the maximal length $L_{\textrm{max}}$ of the demonstrations reaches 10 or above, while TLIP maintains relatively short computation time with increasing $L_{\textrm{max}}$.

\begin{table}[h]
	\centering
	\caption{Computation time for computing the myopic solution.}
	\centering
	\scalebox{0.8}{
	\begin{tabular}{lllll}
		\toprule
		& $L_{\textrm{max}}=5$  & $L_{\textrm{max}}=10$ & $L_{\textrm{max}}=15$ \\
		\midrule
	TLIP       & 3.67 s & 5.29 s & 7.65 s \\ 
    ESMT        & 4.57 s & TO & TO \\
		\bottomrule
	\end{tabular}}
	\label{computation_time}  
	\vspace{-0.05in}
\end{table}  

As ESMT is not scalable, we implement the following four methods for comparison for myopic teaching performances.
\begin{itemize}\denselist
    \item \textbf{AN-TLIP}: TLIP for minimizing AN teaching cost.
    \item \textbf{AL-TLIP}: TLIP for minimizing AL teaching cost.
    \item \textbf{AN-RG}: randomized greedy algorithm for minimizing AN teaching cost. At each iteration, we greedily pick the demonstration with minimal AN teaching cost among a randomly selected subset of demonstrations. 
    \item \textbf{AL-RG}: same with AN-RG except that here we minimize the AL teaching cost.
\end{itemize}
\begin{figure}[!h]
	\centering
	\includegraphics[width=8cm]{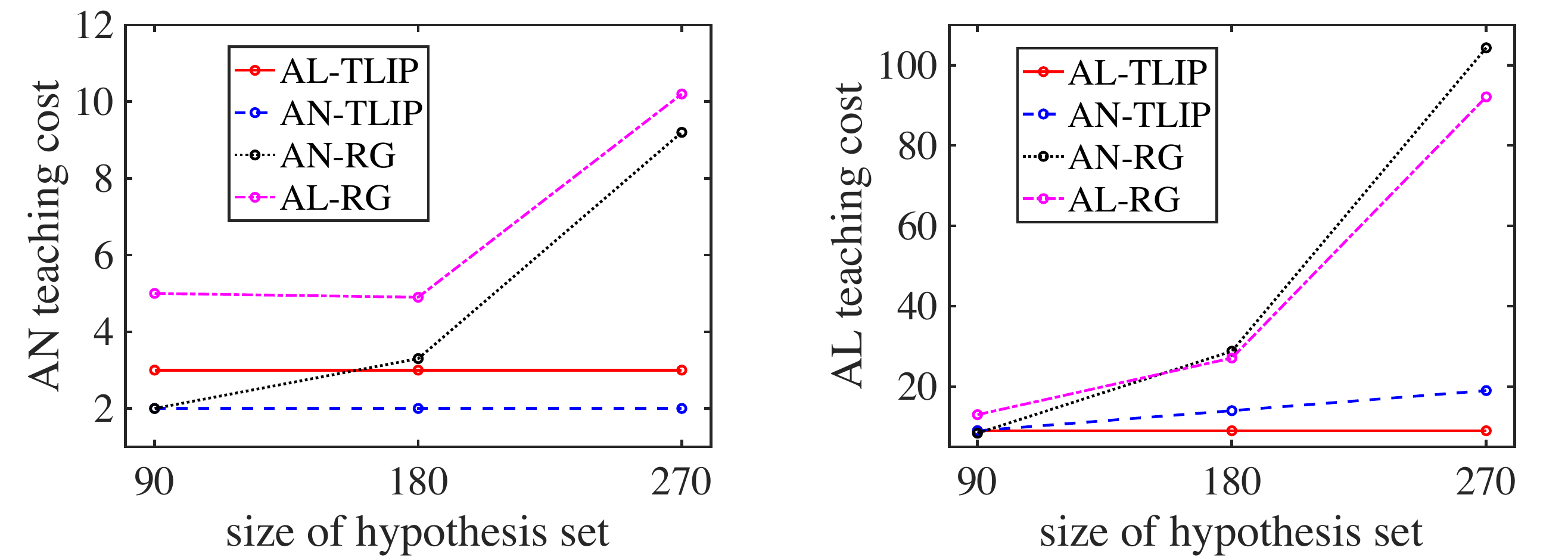}\caption{AN and AL teaching costs under global (uniform) preference with increasing sizes of the hypothesis set.}  
	\label{result1}
	\vspace{-2mm}
\end{figure}		

% The dashed (blue) line, solid (red) line, dotted (black) line and dash-dot (magenta) line indicate AN-TLIP, AL-TLIP, AN-RG and AL-RG, respectively.

\figref{result1} shows that AN-TLIP (resp. AL-TLIP) outperforms the other three methods % are the best for 
when minimizing the AN teaching costs (resp. the AL teaching costs). Specifically, the AN teaching costs using AN-TLIP are 50\%, 78.26\% and 80.39\% less than those using AL-TLIP, AN-RG and AL-RG, respectively. The AL teaching costs using AL-TLIP are 52.63\%, 91.37\% and 90.23\% less than those using AN-TLIP, AN-RG and AL-RG, respectively. \figref{result1} also shows that the growth of the AL teaching cost is more significant with the increasing size of the hypothesis set than that of the AN teaching costs.                               

\paragraph{Teaching with Positive Demonstrations Only} 	
We test the machine teaching algorithm with positive demonstrations only. We consider global preferences, where the learner prefers F-formulas than G-formulas, and with the same temporal operator the learner prefers formula $\phi_1$ than $\phi_2$ if and only if $\phi_1$ implies $\phi_2$. It can be shown that this preference function satisfies the condition of Theorem \ref{prop_pos}.  

\begin{figure}[th]
	\centering
	\includegraphics[width=8cm]{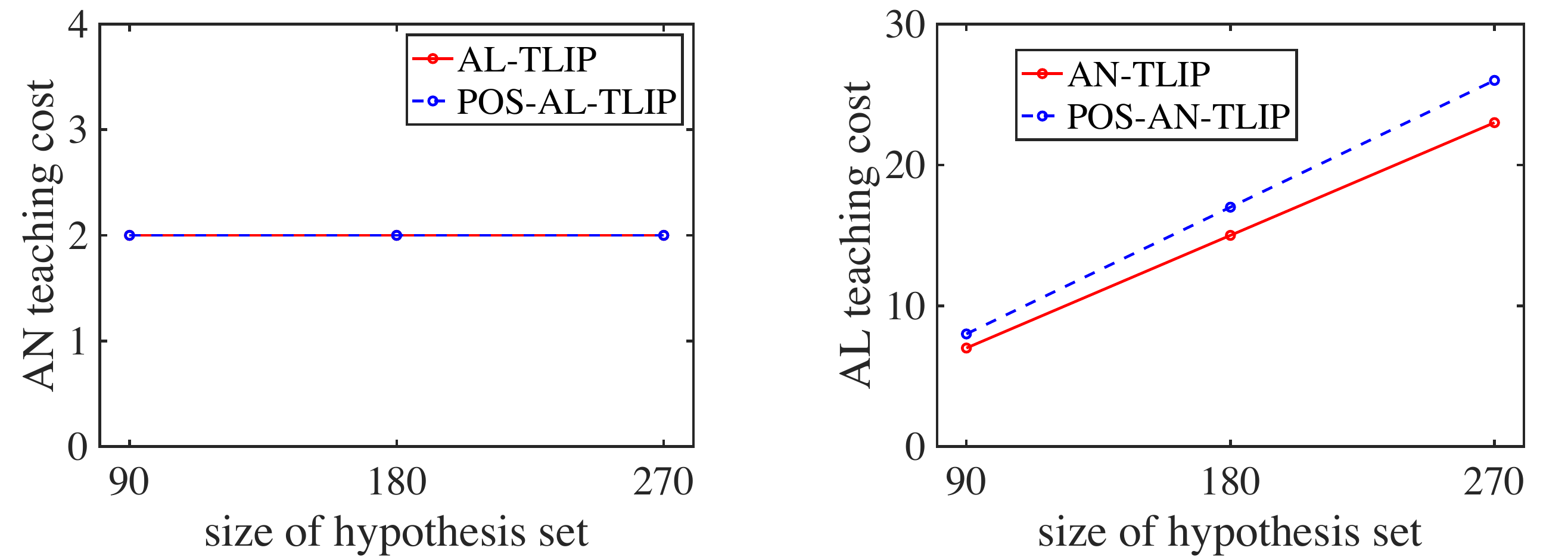}\caption{AN and AL teaching costs with positive demonstrations only and with both positive and negative demonstrations with increasing sizes of the hypothesis set. } 
	\label{result5}
	\vspace{-2mm}
\end{figure}
%\figref{result5} shows that the AN teaching costs of POS-AN-TLIP (i.e., AN-TLIP with positive demonstrations only) are the same as those of AN-TLIP, while the AL teaching costs of POS-AL-TLIP (i.e., AL-TLIP with positive demonstrations only) are up to 20\% more than those of AN-TLIP.
 The corresponding algorithms for AN-TLIP and AL-TLIP with positive demonstrations only are referred to as POS-AN-TLIP and POS-AL-TLIP, respectively. \figref{result5} shows that POS-AN-TLIP and POS-AL-TLIP do not incur much additional teaching cost (up to 20\% more) when restricted to only positive demonstrations.

% the AN and AL teaching costs of using positive demonstrations only are up to 20\% more than the teaching costs of using both positive and negative demonstrations.

% \subsection{Preference-Based Teaching} 
% \label{sec_case_pref}

% For example, if $\phi^{\ast}=G_{\le 5}(x\le 6)$, $\phi^1=F_{\le 3}(x\le 8)$, $\phi^2=F_{\le 3}(x\le 10)$, $\phi^3=F_{\le 3}(x\le 3)$,
% we have $\sigma(\phi^2;\phi^1)<\sigma(\phi^3;\phi^1)<\sigma(\phi^{\ast};\phi^1)$, but $\sigma(\phi^{\ast};\phi^2)<\sigma(\phi^3;\phi^2)$. Therefore, Condition 1 of Theorem \ref{th_bound} does not hold. By setting $\phi^2$ as an intermediate target hypothesis (oracle), the teaching before $\phi^2$ and after $\phi^2$ will both satisfy Condition 1.

% The computed demonstration is

% \begin{figure}[th]
% 	\centering
% 	\includegraphics[width=8cm]{result_pref.pdf}\caption{AN and AL teaching costs with uniform preferences and local preferences. The dashed line and solid line indicate uniform preferences and (non-uniform) global preferences, respectively.} 
% 	\label{result2}  
% \end{figure}		

% \figref{result2} shows the teaching costs with uniform preferences and local preferences, with increasing sizes of the hypothesis set.  

\subsection{Teaching pLTL$_f$ %Temporal Logic 
Formulas under Local Preferences}     	
	\begin{figure}[th]
		\centering
		\includegraphics[width=8cm]{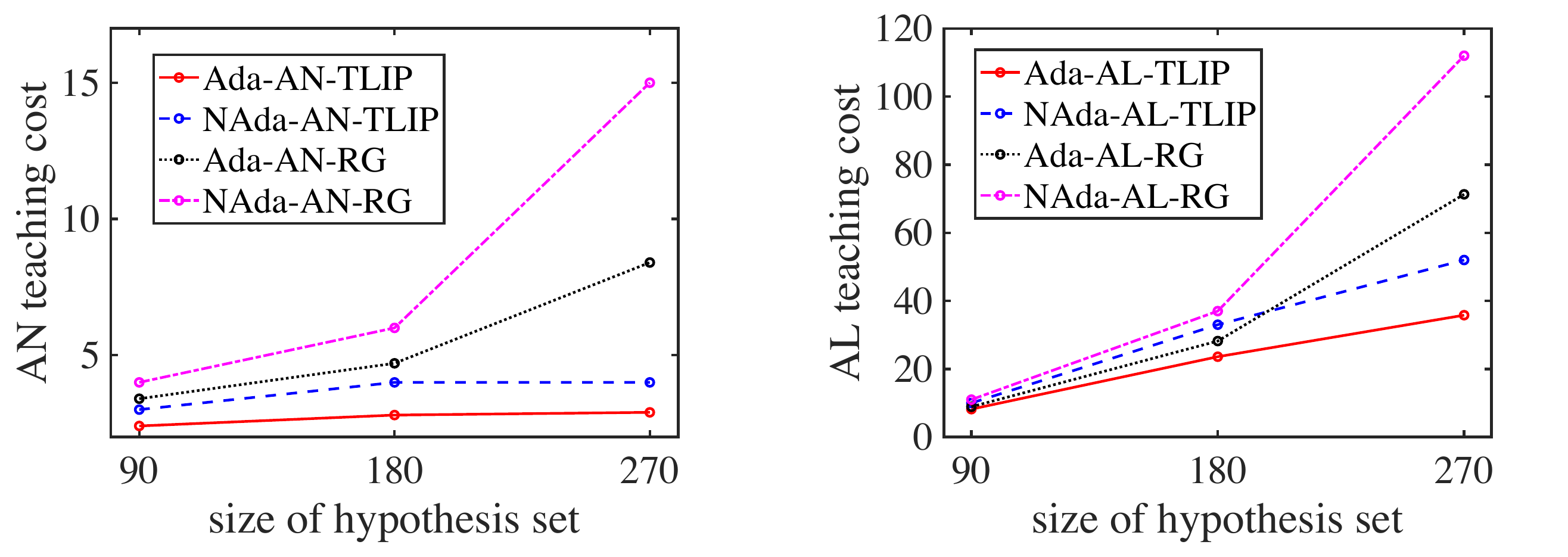}\caption{AN and AL teaching costs for adaptive and non-adaptive teaching with increasing sizes of the hypothesis set. %Ada-AN-TLIP, Ada-AL-TLIP, Ada-AN-RG and Ada-AL-RG represent AN-TLIP, AL-TLIP, AN-RG and AL-RG under the adaptive setting, respectively; NAda-AN-TLIP, NAda-AL-TLIP, Ada-AN-RG and Ada-AL-RG represent under the adaptive setting, respectively
		}
		\label{result3}
		\vspace{-2mm}
	\end{figure}		
% The dashed (blue) line, solid (red) line, dotted (black) line and dash-dot (magenta) line indicate NA-TLIP, Ada-TLIP, NA-RG and Ada-RG, respectively.
% For adaptive teaching, we compute the demonstrations based on the hypothesis pLTL$_f$ formulas of the learner throughout the process.
\looseness -1 \paragraph{Local preferences} For two pLTL$_f$ formulas $\phi_1=F_{\le i_1}(x\le v_1)$ and $\phi_2=F_{\le i_2}(x\le v_2)$, we define the Manhattan distance between $\phi_1$ and $\phi_2$ as $\vert i_1-i_2\vert + \vert v_1-v_2\vert$. We consider the following local
preference:\\
(1) The learner prefers formulas with the same temporal operator as that in the learner's previous hypothesis;\\
(2) With the same temporal operator the learner prefers formulas that are ``closer'' to the learner's previous hypothesis in terms of the Manhattan distance;\\
(3) The learner prefers G-formulas than F-formulas if the learner's current hypothesis are in the form of $F_{\le i}(x\le 0)$ or $F_{\le i}(x\le 10)$ ($i=1,\dots, a$). This is intuitively consistent with human's preferences to switch categories when the values reach certain boundary values. 

% \paragraph{Adaptive teaching under local preferences} 
\paragraph{Adaptive Teaching} To test the advantage of adaptive teaching, we compare it with non-adaptive teaching in the presence of uncertainties. We consider local preferences with added uncertainty noises. Specifically, the learner has equal preference of selecting the formulas in the version space that have the least Manhattan distance from the current hypothesis and also any formula that can be perturbed from these formulas in the version space (here ``perturb'' means adding or subtracting the parameters $i$ or $v$ by 1, e.g., as in $F_{\le i}(x\le v)$). 
% For example, if the formula in the version space that has the least Manhattan distance from the current hypothesis is $\phi_1=G_{\le6}(x\le5)$, while $\phi_2=G_{\le5}(x\le5)$, $\phi_3=G_{\le7}(x\le5)$ and $\phi_4=G_{\le6}(x\le4)$ are also in the version space, $\phi_5=G_{\le6}(x\le6)$ has been eliminated from the version space, then the learner has 25\% probability each of selecting $\phi_1$, $\phi_2$, $\phi_3$ or $\phi_4$. The non-adaptive teaching considers the worst-case scenario based on the initial hypothesis.

% the teaching costs with non-adaptive teaching and adaptive teaching, with increasing sizes of the hypothesis set. 

\figref{result3} shows that Ada-AN-TLIP (i.e., adaptive AN-TLIP) can reduce the AN teaching costs by up to 27.5\% compared with NAda-AN-TLIP (i.e., non-adaptive AN-TLIP), Ada-AN-RG (i.e., adaptive AN-RG) can reduce the AN teaching costs by up to 44\% compared with NAda-AN-RG (i.e., non-adaptive AN-RG), Ada-AL-TLIP (i.e., adaptive AL-TLIP) can reduce the AN teaching costs by up to 31.15\% compared with NAda-AL-TLIP (i.e., non-adaptive AL-TLIP), and Ada-AL-RG (i.e., adaptive AL-RG) can reduce the AN teaching costs by up to 36.34\% compared with NAda-AL-RG (i.e., non-adaptive AL-RG).

% \figref{result3} shows that Ada-AN-TLIP can reduce the AN teaching costs by 27.5\% compared with NAda-AN-TLIP, Ada-AN-RG can reduce the AN teaching costs by 44\% compared with NAda-AN-RG, Ada-AL-TLIP can reduce the AN teaching costs by 31.15\% compared with NAda-AL-TLIP, and Ada-AL-RG can reduce the AN teaching costs by 36.34\% compared with NAda-AL-RG.

\paragraph{Adaptive Teaching with Oracles}   	
\begin{figure}[th]
		\centering
\includegraphics[width=8cm]{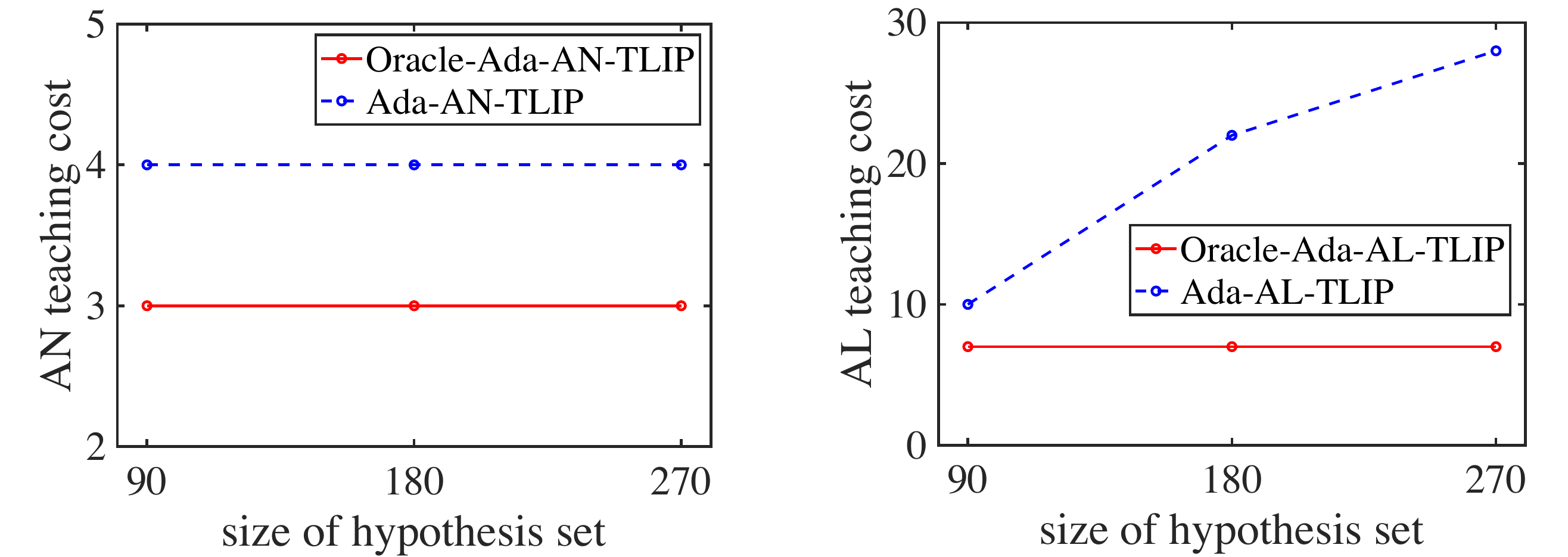}\caption{AN and AL teaching costs for adaptive teaching with oracles and without oracles with increasing sizes of the hypothesis set.} 
	\label{result4}
	\vspace{-2mm}
\end{figure}		
% The dashed line and solid line indicate TLIP with oracle and without oracle, respectively.
To decompose the teaching problem into subproblems that satisfy Condition 1 of Theorem \ref{th_bound}, we design the oracle which outputs an intermediate target hypothesis $F_{\le1}(x\le10)$\footnote{Theoretically, the oracle can be designed to output any formula of the form $F_{\le i}(x\le 0)$ or $F_{\le i}(x\le 10)$ ($i=1,\dots, a$).}. %can be selected as the oracle .
\figref{result4} shows that Oracle-Ada-AN-TLIP (adaptive AN-TLIP with oracles) can reduce the AN teaching costs by 25\% compared with Ada-AN-TLIP and Oracle-Ada-AL-TLIP (adaptive AL-TLIP with oracles) can reduce the AL teaching costs by up to 75\% compared with Ada-AL-TLIP.

% Theoretically, the oracle can output any formula in the form of $F_{\le i}(x\le 0)$ or $F_{\le i}(x\le 10)$ ($i=1,\dots, a$). %can be selected as the oracle .

\section{Conclusion}                                                                       
\label{conclusion}
We presented the first attempt for teaching parametric linear temporal logic formulas to a learner with preferences. We also explored how to more efficiently teach the learner utilizing adaptivity and oracles. The results show the effectiveness of the proposed approach. We believe this is an important step towards practical algorithms for teaching more complex concept classes in the real-world scenarios. %the first attempt of machine teaching of temporal logic formulas.
% The same methodology can be used in teaching other types of temporal logic formulas. 
For future work, we will explore teaching methods for more general forms of temporal logic formulas, with more specific learning algorithms for inferring temporal logic formulas. 

% For some cases, the learner could have a specific learning algorithm to infer temporal logic formulas from data. 

\bibliographystyle{IEEEtran}
\bibliography{zheref}

\section*{APPENDIX}
\section{Teaching pLTL$_f$ Formulas to Learners with Global Preferences}
\subsection{Formulation of Optimal Teaching}

We formulate an integer programming problem for optimally teaching a pLTL$_f$ formula to a learner with global preferences. We define the \textit{preferred hypothesis set} as $\tilde{\Phi}:=\{\phi\in\Phi:\sigma(\phi;\cdot)\le\sigma(\phi^{\ast}; \cdot)\}$, i.e., the set of hypotheses that are more preferred by the learner than the target hypothesis.
For a trajectory $\rho_L$ and a pLTL$_f$ formula $\phi$, we denote $c(\phi, \rho_L)=1$ if $(\rho_L,0)\models_{\rm{S}} \phi$; $c(\phi, \rho_L)=-1$ if $(\rho_L,0)\models_{\rm{S}} \lnot\phi$; and $c(\phi, \rho_L)=0$ if $(\rho_L,0)\not\models_{\rm{S}} \phi$ and $(\rho_L,0)\not\models_{\rm{S}} \lnot\phi$.  
We denote the set of possible demonstrations with length at most $L_{\textrm{max}}$ as $D_{L_{\textrm{max}}}=\{\rho_{L_1}^1, \rho_{L_2}^2, \dots, \rho_{L_{\vert D_{L_{\textrm{max}}}\vert}}^{\vert D_{L_{\textrm{max}}}\vert}\}$.  One way to optimally solve the machine teaching problem is to enumerate all possible demonstrations in $D_{L_{\textrm{max}}}$ and solve  the following integer programming problem. 
\begin{align}
\begin{split}
\min\limits_{\xi_i} ~ & \sum\limits_{1\le i\le \vert D_{L_{\textrm{max}}}\vert}\chi                           \\
\text{subject to:} ~ & \xi_i\in \{0,1\}, \forall i\in\{1, 2, \dots, \vert D_{L_{\textrm{max}}}\vert\}, \\
& \sum_{k=1}^{\vert D_{L_{\textrm{max}}}\vert}b_{j,k}\gamma_j\xi^k\ge 1,  \forall j, \textrm{s.t.}, \phi_j\in\tilde{\Phi}, 
\end{split}
\label{MIP}               
\end{align} 
where $\chi=\xi_i$ for AN teaching costs and $\chi=\xi_iL_i$ (where $L_i$ is the length of the $i$-th demonstration in $D_{L_{\textrm{max}}}$) for AL teaching costs, $b_{j,k}=1$ if either $c(\phi^{\ast},\rho_{L_k}^k)= 1$ and $c(\phi^j,\rho_{L_k}^k)= -1$, or $c(\phi^{\ast},\rho_{L_k}^k)= -1$ and $c(\phi^j,\rho_{L_k}^k)= 1$, and $b_{j,k}=0$ otherwise. As $c(\phi^j,\rho_{L_k}^k)$ is known by evaluating the satisfaction of $\rho_{L_k}^k$ with respect to $\phi^j$, $b_{j,k}$ is also a known variable. Therefore, problem (\ref{MIP}) is an integer programming problem with the worst-case time complexity in the order of $2^{\vert D_{L_{\textrm{max}}}\vert}$. As $\vert D_{L_{\textrm{max}}}\vert=\sum_{i=1}^{L_{\textrm{max}}}\vert S\vert^{L_{\textrm{max}}}$, $2^{\vert D_{L_{\textrm{max}}}\vert}$ is doubly exponential with respect to the maximal length of the demonstrations; hence the optimization is intractable to solve. Therefore, we resort to greedy methods for myopic teaching with near-optimal performance in \secref{sec_TLIP}.

\subsection{Teachable Hypotheses}
For a global preference function $\sigma$, we say $\phi^{\ast}$ is \textit{teachable} from a sequence $D$ of demonstrations if and only if $\forall \phi\in\Phi(D)\setminus\{\phi^{\ast}\}, \sigma(\phi; \cdot)>\sigma(\phi^{\ast}; \cdot)$.

\begin{corollary}
	Given a target pLTL$_f$ formula $\phi^{\ast}$, the hypothesis set $\Phi=\{\phi_1, \phi_2, \dots, \phi_n\}$ and global preference function $\sigma$, if $\phi^{\ast}\in\Phi$ is teachable from a sequence of demonstrations $D=\{(\rho^1_{L_1}, l_1), (\rho^2_{L_2}, l_2), \dots, (\rho^{N_D}_{L_{N_D}}, l^{N_D}_{L_{N_D}})\}$, then 
	$\max\limits_{1\le k\le N_D}L_k\ge\max\limits_{\phi_i\in \tilde{\Phi}}\min\{\zeta(\phi_i, 1), \zeta(\phi_i, -1)\}$, where $\tilde{\Phi}:=\{\phi\in\Phi~\vert~\sigma(\phi;\cdot)\le\sigma(\phi^{\ast}; \cdot)\}$ is the preferred hypothesis set.
	\label{corollary}                 	    	 
\end{corollary}

\noindent\textbf{Proof of Corollary \ref{corollary}}:\\
Each demonstration $\rho^k_{L_k}$ is strongly inconsistent with a subset $\hat{\Phi}_k$ of pLTL$_f$ formulas in $\Phi$ ($\hat{\Phi}_k$ may be empty). If $\phi^{\ast}\in\Phi$ is teachable from a sequence of demonstrations $D=\{(\rho^1_{L_1}, l_1), (\rho^2_{L_2}, l_2), \dots, (\rho^{N_D}_{L_{N_D}}, l^{N_D}_{L_{N_D}})\}$, then it holds that $\tilde{\Phi}\setminus\{\phi^{\ast}\}\subset\bigcup\limits_{k}\hat{\Phi}_k$. From Theorem \ref{gamma0}, for each $k$, we have $L_k\ge \max\{\zeta(\phi^{\ast}, l), \max\limits_{1\le i\le \vert\hat{\Phi}_k\vert}\zeta(\phi_i, -l)\}$. Therefore, we have
\begin{align}\nonumber
\begin{split}
\max\limits_{1\le k\le N_D}L_k&\ge\max\limits_{1\le k\le N_D}\max\{\zeta(\phi^{\ast}, l), \max\limits_{1\le i\le \vert\hat{\Phi}_k\vert}\zeta(\phi_i, -l)\}\\
&\ge\max\limits_{\phi_i\in \tilde{\Phi}}\min\{\zeta(\phi_i, 1), \zeta(\phi_i, -1)\}.  
\end{split}
\end{align}

\section{Theoretical Proofs}

\noindent\textbf{Proof of Theorem \ref{gamma0}}:\\
	To prove Theorem \ref{gamma0}, we first prove the following lemma.
		\begin{lemma} 
		For a demonstration $(\rho_{L}, l)$ and the minimal time length $\zeta(\phi,l)$, if $L< t+\zeta(\phi,l)$, then if $l=1$, $(\rho_{L},t)\models_{\rm{W}}\lnot\phi$; and if $l=-1$, $(\rho_{L},t)\models_{\rm{W}}\phi$.\\
		\label{gamma2}
	\end{lemma}		
	
% \noindent\textbf{Proof of Lemma \ref{gamma2}}:
\begin{proof}	
	\noindent We use induction to prove Lemma \ref{gamma2}. \\
	(i) We first prove that Lemma \ref{gamma2} holds for atomic predicate $\pi$. As $\zeta(\pi,l)=0$, if $L<t+\zeta(\phi,l)$, then $L<t$, so according to Definition \ref{weak}, for $l=1$ and $l=-1$, $(\rho_{L},t)\models_{\rm{W}}\pi$ and $(\rho_{L},t)\models_{\rm{W}}\lnot\pi$, Lemma \ref{gamma2} trivially holds.
	
	(ii) We assume that Lemma \ref{gamma2} holds for $\phi$ and prove Lemma \ref{gamma2} holds for $\lnot\phi$. If Lemma \ref{gamma2} holds for $\phi$, then if $L<t+\zeta(\phi,l)$, we have for $l=1$, $(\rho_{L},t)\models_{\rm{W}}\lnot\phi$; for $l=-1$, $(\rho_{L},t)\models_{\rm{W}}\phi$. Thus if $L<t+\zeta(\phi,-l)$, we have for $l=-1$, $(\rho_{L},t)\models_{\rm{W}}\lnot\phi$; for $l=1$, $(\rho_{L},t)\models_{\rm{W}}\phi$.
	Therefore, 
	if $$L<t+\zeta(\lnot\phi,l)=t+\zeta(\phi, -l),$$ we have for $l=1$,  $(\rho_{L},t)\models_{\rm{W}}\phi$, i.e. $(\rho_{L},t)\models_{\rm{W}}\lnot(\lnot\phi)$; for $l=-1$, $(\rho_{L},t)\models_{\rm{W}}\lnot\phi$. Therefore, Lemma \ref{gamma2} holds for $\lnot\phi$.
	
	(iii) We assume that Lemma \ref{gamma2} holds for $\phi_1,\phi_2$ and prove Lemma \ref{gamma2} holds for $\phi_1\wedge\phi_2$ and $\phi_1\vee\phi_2$. 
	
	For $l=1$, if Lemma \ref{gamma2} holds for $\phi_1$ and $\phi_2$, then if $L<t+\zeta(\phi_1,l)$, we have $(\rho_{L},t)\models_{\rm{W}}\lnot\phi_1$; if $L<\zeta(\phi_2,l)$, we have $(\rho_{L},t)\models_{\rm{W}}\lnot\phi_2$. 
	
	If $$L<t+\max\{\zeta(\phi_1, l), \zeta(\phi_2, l)\},$$ then $L<t+\zeta(\phi_1, l)$ or $L<t+\zeta(\phi_2, l)$, thus $(\rho_{L},t)\models_{\rm{W}}\lnot\phi_1$ or $(\rho_{L},t)\models_{\rm{W}}\lnot\phi_2$. So we have $$(\rho_{L},t)\models_{\rm{W}}\lnot\phi_1\vee\lnot\phi_2,$$ hence we have $(\rho_{L},t)\models_{\rm{W}}\lnot(\phi_1\wedge\phi_2)$.  
	
	For $l=-1$, if Lemma \ref{gamma2} holds for $\phi_1$ and $\phi_2$, then if $L<t+\zeta(\phi_1,l)$, we have $(\rho_{L},t)\models_{\rm{W}}\phi_1$; if $L<t+\zeta(\phi_2,l)$, we have $(\rho_{L},t)\models_{\rm{W}}\phi_2$. If $L<t+\min\{\zeta(\phi_1, l), \zeta(\phi_2, l)\}$, then $L<t+\zeta(\phi_1, l)$ and $L<t+\zeta(\phi_2, l)$, thus $(\rho_{L},t)\models_{\rm{W}}\phi_1$ and $(\rho_{L},t)\models_{\rm{W}}\phi_2$. Therefore, we have $$(\rho_{L},t)\models_{\rm{W}}\phi_1\wedge\phi_2.$$ Therefore, Lemma \ref{gamma2} holds for $\phi_1\wedge\phi_2$.    		
	
	Similarly, it can be proved by induction that Lemma \ref{gamma2} holds for $\phi_1\vee\phi_2$. 	
	
	(iv) We assume that Lemma \ref{gamma2} holds for $\phi$ and prove Lemma \ref{gamma2} holds for $F_{\le\tau}\phi$ and $G_{\le\tau}\phi$.	
	
	For $l=1$, if $L<t+\zeta(\phi, l)+t_1$, then for any $t'\in[t, t+\tau]$, we have $L<\zeta(\phi, l)+t'$, so if Lemma \ref{gamma2} holds for $\phi$, then for any $t'\in[t, t+\tau]$, we have $(\rho_{L},t')\models_{\rm{W}}\lnot\phi$, thus $$(\rho_{L},t)\models_{\rm{W}}G_{\le\tau}\lnot\phi,$$ hence we have $(\rho_{L},t)\models_{\rm{W}}\lnot F_{\le\tau}\phi$. 
	
	For $l=-1$, if $L<t+\zeta(\phi, l)+\tau$, then there exists $t'\in[t, t+\tau]$ such that $L<\zeta(\phi, l)+t'$, so if Lemma \ref{gamma2} holds for $\phi$, we have $(\rho_{L},t')\models_{\rm{W}}\phi$, thus $$(\rho_{L},t)\models_{\rm{W}}F_{\le\tau}\phi.$$  Therefore, Lemma \ref{gamma2} holds for $F_{\le\tau}\phi$.    		
	
	Similarly, it can be proved by induction that Lemma \ref{gamma2} holds for $G_{\le\tau}\phi$.
	
	Therefore, it is proved by induction that Lemma \ref{gamma2} holds for any pLTL$_f$ formula $\phi$.
	\end{proof}
	Now we proceed to prove Theorem \ref{gamma0}. According to Definition \ref{perfect}, if a demonstration $(\rho_{L}, l)$ is strongly inconsistent with a subset $\hat{\Phi}=\{\phi_{i}\}^{\hat{N}}_{i=1}\subset\Phi$ of pLTL$_f$ formulas, then either $l=1$, $(\rho_{L},0)\models_{\rm{S}}\phi^{\ast}$, $(\rho_{L},0)\models_{\rm{S}}\lnot\phi_{i}$ for any $\phi_i\in\hat{\Phi}$, or $l=-1$,
	$(\rho_{L},0)\models_{\rm{S}}\lnot\phi^{\ast}$, $(\rho_{L},0)\models_{\rm{S}}\phi_{i}$ for any $\phi_i\in\hat{\Phi}$. 
	
    From Lemma \ref{gamma2}, if $L<\zeta(\phi^{\ast}, l)$, then 
    either $l=1$, $(\rho_{L},0)\models_{\rm{W}}\lnot\phi^{\ast}$ holds, i.e., $(\rho_{L},0)\models_{\rm{S}}\phi^{\ast}$ does not hold, or $l=-1$, $(\rho_{L},0)\models_{\rm{W}}\phi^{\ast}$ holds, i.e., $(\rho_{L},0)\models_{\rm{S}}\lnot\phi^{\ast}$ does not hold. Therefore, the demonstration $(\rho_{L}, l)$ is not strongly inconsistent with a subset $\hat{\Phi}=\{\phi_{i}\}^{\hat{N}}_{i=1}\subset\Phi$ of pLTL$_f$ formulas. Contradiction! Hence we have $L\ge\zeta(\phi^{\ast}, l)$.

    Similarly, from Lemma \ref{gamma2}, for any $\phi_i\in\hat{\Phi}$, if $L<\zeta(\phi_i, -l)$, then 
    either $l=1$, $(\rho_{L},0)\models_{\rm{W}}\phi_i$ holds, i.e., $(\rho_{L},0)\models_{\rm{S}}\lnot\phi_i$ does not hold, or $l=-1$, $(\rho_{L},0)\models_{\rm{W}}\lnot\phi_i$ holds, i.e., $(\rho_{L},0)\models_{\rm{S}}\phi_i$ does not hold. Therefore, the demonstration $(\rho_{L}, l)$ is not strongly inconsistent with a subset $\hat{\Phi}=\{\phi_{i}\}^{\hat{N}}_{i=1}\subset\Phi$ of pLTL$_f$ formulas. Contradiction! Hence we have $L\ge\zeta(\phi_i, l)$ for any $\phi_i\in\hat{\Phi}$.
    
    Therefore, we have $L\ge\zeta(\phi^{\ast}, l)$, and $L\ge\zeta(\phi_i, l)$ for any $\phi_i\in\hat{\Phi}$, hence $L\ge \max\{\zeta(\phi^{\ast}, l), \max\limits_{1\le i\le \hat{N}}\zeta(\phi_i, -l)\}$.\\

\noindent\textbf{Proof of Theorem \ref{prop_pos}}:\\
If there exists a pLTL$_f$ formula $\phi$ such that $\phi^{\ast}\Rightarrow\phi'$ and $\sigma(\cdot,\phi')<\sigma(\cdot,\phi^{\ast})$, then for any demonstration $(\rho_L, 1)$, if $(\rho_L, 0)\models_{\rm{S}}\phi^{\ast}$, then $(\rho_L, 0)\models_{\rm{S}}\phi'$. However, in order to teach $\phi^{\ast}$ using a sequence $D_{\textrm{p}}$ of positive demonstrations, there should be one demonstration $(\rho'_L, 1)$ that is strongly inconsistent with $\phi'$, as  
$\sigma(\cdot,\phi')<\sigma(\cdot,\phi^{\ast})$. Therefore, $(\rho'_L, 0)\models_{\rm{S}}\lnot\phi'$. Contradiction!\\

For a sequence of demonstrations $D_{\textrm{p}}=\{(\rho^1_{L_1}, 1), (\rho^2_{L_2}, 1), \dots, (\rho^{N_D}_{L_{N_D}}, 1)\}$, each demonstration $\rho^k_{L_k}$ is strongly inconsistent with a subset $\hat{\Phi}_k$ of pLTL$_f$ formulas in $\Phi$ ($\hat{\Phi}_k$ may be empty). If $\phi^{\ast}\in\Phi$ is teachable from  $D_{\textrm{p}}$, then it holds that $\tilde{\Phi}\setminus\{\phi^{\ast}\}\subset\bigcup\limits_{k}\hat{\Phi}_k$. From Theorem \ref{gamma0}, for each $k$, we have $L_k\ge \max\{\zeta(\phi^{\ast}, 1), \max\limits_{1\le i\le \vert\hat{\Phi}_k\vert}\zeta(\phi_i, -1)\}$. Therefore, we have
\begin{align}\nonumber
\begin{split}
\max\limits_{1\le k\le\vert D_{\textrm{p}}\vert}L_k&\ge\max\limits_{1\le k\le \vert D_{\textrm{p}}\vert}\max\{\zeta(\phi^{\ast}, 1), \max\limits_{1\le i\le \vert\hat{\Phi}_k\vert}\zeta(\phi_i, -1)\}\\
&\ge\max\{\zeta(\phi^{\ast}, 1), \max\limits_{\phi_i\in\tilde{\Phi}\setminus\{\phi^{\ast}\}}\zeta(\phi_i, -1)\}\\
&\ge\max\limits_{\phi_i\in\tilde{\Phi}\setminus\{\phi^{\ast}\}}\zeta(\phi_i, -1).  
\end{split}
\end{align}

\noindent\textbf{Proof of Theorem \ref{th_bound}}:
\paragraph{Global preferences} When the preference function is global, the myopic adaptive TLIP algorithm reduces to a greedy set cover algorithm for minimizing the AN teaching cost. Therefore, we have
	\[
    \begin{split}
    \textrm{AN-Cost}^{\textrm{WC}}_{\mathcal{S}}(\mathcal{T}^{\textrm{m}},~& \mathcal{L}_{\sigma})\le\\
    &(\log \vert\tilde{\Phi}_{\mathcal{S}}\vert+1)\textrm{AN-Complexity}_{\mathcal{S}}(\mathcal{L}_{\sigma}).
    \end{split}
    \]
For minimizing the AL teaching cost, the myopic adaptive TLIP algorithm reduces to a greedy weighted set cover algorithm with the weights being the lengths of the demonstrations. Therefore, we have
	\[
    \begin{split}
    \textrm{AL-Cost}^{\textrm{WC}}_{\mathcal{S}}(\mathcal{T}^{\textrm{m}},~& \mathcal{L}_{\sigma})\le\\
    &(\log \vert\tilde{\Phi}_{\mathcal{S}}\vert+1)\textrm{AL-Complexity}_{\mathcal{S}}(\mathcal{L}_{\sigma}).
    \end{split}
    \]

% \clearpage
% \section{Proof of Theorem \ref{th_bound}}\label{app:proof:greedy:suff}
% % In this section, we provide the proof for Theorem \ref{th_bound}
% In this section, % we introduce useful notation and formally define the cost of a teaching algorithm. Then 
% we prove the upper bound of the greedy cost as presented in Theorem \ref{th_bound}.
% % Proof outline:
% % \begin{itemize}\denselist
% % \item Greedy picks a demonstration with the maximal reduction in $\futurecostapprox$
% % \item Reduction to setcover if the demonstration does not eliminate learner's current hypothesis $\hypothesis^k$?
% % \item It can only happen a limited number of times that a greedy demonstration removes $\hypothesis^k$.
% % \end{itemize}

\paragraph{Local preferences}  To be consistent with the notations used in the main paper, for any integer $k\in[1,K]$, we use $\phi^{k}$ to denote the learner's hypothesis formula after $k$ demonstrations, and $\Phi^k$ to denote the version space after $k$ demonstrations. 

% Let us use $\pi$ to denote an adaptive teaching algorithm and $\randomstate$ to denote the internal randomness of the learner. 
% We fix $\randomstate$, the preference $\ordering$, the target hypothesis $\hstar$, the learner's initial hypothesis $\hinit$, and the initial version space $\Hypotheses$. 
% We denote the learner's hypothesis after running $\pi^k$ (i.e., running $\pi$ for $t$ steps) as $\hypothesis(\pi^k, \randomstate, \ordering, \hstar, \hinit, \Hypotheses)$. To be consistent with the notations used in the main paper, we use $h^k$ as the shorthand notation for $\hypothesis(\pi^k, \randomstate, \ordering, \hstar, \hinit, \Hypotheses)$ whenever it is unambiguous.

After $k$ demonstrations, we define the \emph{preferred version space}
\begin{align}
  \prefset_{\ordering,\hstar}( \hypothesis^k, \Hypotheses^k) := \{\hypothesis' \in \Hypotheses^k: \orderingof{\hypothesis'}{\hypothesis^k} \leq \orderingof{\hstar}{\hypothesis^k}\}\label{eq:prefset}
\end{align}
to be the set of hypotheses in the current version space $\Hypotheses^k$ that are preferred over $\hstar$ (according to $\ordering$) from $\hypothesis^k$.

  Assume that the preference function $\ordering$ and the structure of tests satisfy Condition 1 and 2 from Theorem \ref{th_bound}. Suppose we have run the myopic adaptive TLIP algorithm for $m$ demonstrations. Let $\demonstration^k=(\rho^k_{L_k}, l^k)$ be the current demonstration, $\demonstrations^k = \{\demonstration^0, \dots, \demonstration^{k-1}\}$ be the sequence of demonstrations chosen by the greedy teacher up to $k$, and $\hypothesis^k$ be the learner's current hypothesis. For any given demonstration $\demonstration$, let $\hypothesis_z$ be the learner's next hypothesis assuming the teacher provides $\demonstration$. Then, we will prove the following inequality:
  \begin{align} 
    & |\prefset_{\ordering,\hstar}( \hypothesis^k,  \Hypotheses(\demonstrations^k))| - |\prefset_{\ordering,\hstar}( \hypothesis^{k+1},  \Hypotheses(\demonstrations^k \cup \{\demonstration^k\}))| \notag \\
    \geq ~&\frac12 \max_\demonstration \left( |\prefset_{\ordering,\hstar}( \hinit,  \Hypotheses(\demonstrations^k))|  - |\prefset_{\ordering,\hstar}( \hinit,  \Hypotheses(\demonstrations^k \cup \{\demonstration\}))| \right).
    \label{lemma_ineq}
  \end{align}
  % $\forall \demonstration, |\prefset_{\ordering,\hstar}( \hypothesis_m,  \Hypotheses(\demonstrations_m \cup \{\demonstration\})| \leq |\prefset(\uordering, \hstar, \hypothesis^0,  \Hypotheses(\demonstrations_m \cup \{\demonstration\})|$.
% \end{lemma}

% \begin{proof}
  The myopic adaptive TLIP algorithm picks the demonstration which leads to the smallest preferred version space. That is,
  \begin{align*}
    \demonstration^k = \argmin_\demonstration |\prefset_{\ordering,\hstar}( \hypothesis_z, \Hypotheses(\demonstrations^k \cup \{\demonstration\}))|.
  \end{align*}
  Here, $\hypothesis_z$ denotes the hypothesis that the learner takes if the learner observes demonstration $z$ after $k$ demonstrations. If $\demonstration$ is inconsistent
  with $\hypothesis^k$, then by Condition 2 of
  Theorem \ref{th_bound}, there exists a demonstration $\demonstration'$ which is
  consistent with $\hypothesis^k$ and only differs from $\demonstration$ at
  $\hypothesis^k$, i.e.,
  $\Hypotheses(\{\demonstration'\})\setminus \Hypotheses(\{\demonstration\}) =
  \{\hypothesis^k\}$.
   Then we have
    \begin{align}
    &  \prefset_{\ordering,\hstar}( \hypothesis^k, \Hypotheses(\demonstrations^k \cup \{\demonstration\})) \notag \\
    \stackrel{}{=}~&  \{\hypothesis' \in \Hypotheses(\demonstrations^k \cup \{\demonstration\}): \orderingof{\hypothesis'}{\hypothesis^k} \leq \orderingof{\hstar}{\hypothesis^k}\} \notag \\
    \stackrel{(a)}{=}~&  \{\hypothesis' \in\Hypotheses(\demonstrations^k \cup \{\demonstration\}): \orderingof{\hypothesis_z}{\hypothesis^k} \leq \orderingof{\hypothesis'}{\hypothesis^k} \leq \orderingof{\hstar}{\hypothesis^k}\} \notag \\
    \stackrel{(b)}{\subseteq}~&  \{\hypothesis' \in\Hypotheses(\demonstrations^k \cup \{\demonstration\}): \orderingof{\hypothesis'}{\hypothesis_z} \leq \orderingof{\hstar}{\hypothesis_z}\} \notag \\
    \stackrel{}{=}~&  \prefset_{\ordering,\hstar}( \hypothesis_z, \Hypotheses(\demonstrations^k \cup \{\demonstration\})), \label{eq:prefset:ht}
  \end{align}
  where step (a) follows from the fact that $\Hypotheses(\demonstrations^k \cup \{\demonstration\}) \cap \prefset_{\ordering,\hypothesis_z}(\hypothesis^k, \Hypotheses(\demonstrations^k)) = \hypothesis_z$, and step (b) follows from Condition~2 of Theorem \ref{th_bound}.
From (\ref{eq:prefset:ht}) we know $\prefset_{\ordering,\hstar}( \hypothesis^k, \Hypotheses(\demonstrations^k \cup \{\demonstration\})) \subseteq \prefset_{\ordering,\hstar}( \hypothesis_\demonstration, \Hypotheses(\demonstrations^k \cup \{\demonstration\}))$.
  % \begin{align*}
  %   \prefset_{\ordering,\hstar}( \hypothesis^k, \Hypotheses(\demonstrations^k \cup \{\demonstration\})) \subseteq \prefset_{\ordering,\hstar}( \hypothesis_\demonstration, \Hypotheses(\demonstrations^k \cup \{\demonstration\}))
  %   \label{eq:greedy_jump_extend}
  % \end{align*}
  Therefore,
  \begin{align}
    |\prefset_{\ordering,\hstar}( \hypothesis^k, \Hypotheses(\demonstrations^k \cup \{\demonstration'\}))| - 1 &=
    |\prefset_{\ordering,\hstar}( \hypothesis^k, \Hypotheses(\demonstrations^k \cup \{\demonstration\}))| \notag\\&\leq |
    \prefset_{\ordering,\hstar}( \hypothesis_\demonstration, \Hypotheses(\demonstrations^k \cup \{\demonstration\}))|,\label{eq:hzvshzp}
  \end{align}
  which gives us $|\prefset_{\ordering,\hstar}( \hypothesis^k, \Hypotheses(\demonstrations^k \cup \{\demonstration'\}))| \leq
  \prefset_{\ordering,\hstar}( \hypothesis_\demonstration, \Hypotheses(\demonstrations^k \cup \{\demonstration\}))| + 1$.

  We consider the following three cases.
  \begin{enumerate}[I.]
  \item \label{enum:suff:case1} $  |\prefset_{\ordering,\hstar}( \hypothesis^k, \Hypotheses(\demonstrations^k \cup \{\demonstration'\}))| =
    |\prefset_{\ordering,\hstar}( \hypothesis_\demonstration, \Hypotheses(\demonstrations^k \cup \{\demonstration\}))| + 1$. Then, by Eq.~\eqref{eq:hzvshzp}, we have
    $$\prefset_{\ordering,\hstar}( \hypothesis^k, \Hypotheses(\demonstrations^k \cup \{\demonstration\})) = \prefset_{\ordering,\hstar}( \hypothesis_\demonstration, \Hypotheses(\demonstrations^k \cup \{\demonstration\})).$$
    That is, even the demonstration $\demonstration$ can bring the learner to a new hypothesis $\hypothesis_\demonstration$, it does \emph{not} introduce new hypotheses into the preferred version space.

  \item \label{enum:suff:case2} $  |\prefset_{\ordering,\hstar}( \hypothesis^k, \Hypotheses(\demonstrations^k \cup \{\demonstration'\}))| <
    \prefset_{\ordering,\hstar}( \hypothesis_\demonstration, \Hypotheses(\demonstrations^k \cup \{\demonstration\}))| + 1$. In this case, the greedy teacher will not pick $\demonstration$, because the gain of demonstration $\demonstration'$ is no less than the gain of $\demonstration$ in terms of the greedy heuristic. In the special case where $|\prefset_{\ordering,\hstar}( \hypothesis^k, \Hypotheses(\demonstrations^k \cup \{\demonstration'\}))| =
    \prefset_{\ordering,\hstar}( \hypothesis_\demonstration, \Hypotheses(\demonstrations^k \cup \{\demonstration\}))|$, the teacher does not pick $\demonstration$, because it makes the learner move away from its current hypothesis and hence is less preferred.
  \end{enumerate}

  For completeness, we also consider the case when the demonstration $\demonstration$ is consistent with $\hypothesis^k$:
   
 \begin{enumerate}[I.]
     %[III.]
     \addtocounter{enumi}{2}
  \item \label{enum:suff:case3} If the teacher picks a consistent demonstration $\demonstration$, then the learner does not move away from her current hypothesis $\hypothesis^k$. As a result, the preference ordering among set $\prefset_{\ordering,\hstar}( \hypothesis^k, \Hypotheses(\demonstrations^k \cup \{\demonstration'\}))$ remains the same.
  \end{enumerate}

  With the above three cases set up, we now reason about the gain of the myopic algorithm. An important observation is that, the greedy teaching demonstrations never add any hypotheses into the preferred version space. Therefore, after $k$ demonstrations, for any demonstration $\demonstration$, we have
  \begin{align}
    \label{eq:prefvschange}
    \prefset_{\ordering,\hstar}( \hypothesis^k, \Hypotheses(\demonstrations^k \cup \{\demonstration\}))
    &= \prefset_{\ordering,\hstar}( \hypothesis_{t-1}, \Hypotheses(\demonstrations^k \cup \{\demonstration\})) \notag \\
    &= \dots \notag \\
    &= \prefset_{\ordering,\hstar}( \hypothesis_{0}, \Hypotheses(\demonstrations^k \cup \{\demonstration\})).
  \end{align}

  Next, we look into the gain for each of the three cases above.
  \begin{enumerate}[I.]
  \item Adding $\demonstration^k$ changes the learner's hypothesis, i.e., $\hypothesis^{k+1} \neq \hypothesis^k$, but the resulting preferred version space induced by $\hypothesis^{k+1}$ is the same with that of $\hypothesis^k$. In this case,
    \begin{align}
      \label{eq:greedy_case1}
      & |\prefset_{\ordering,\hstar}( \hypothesis^k,  \Hypotheses(\demonstrations^k))| - |\prefset_{\ordering,\hstar}( \hypothesis^{k+1},  \Hypotheses(\demonstrations^k \cup \{\demonstration^k\}))| \notag \\
      = ~&|\prefset_{\ordering,\hstar}( \hinit,  \Hypotheses(\demonstrations^k))| \notag\\&~~~~~ - \min_\demonstration |\prefset_{\ordering,\hstar}( \hypothesis^{k+1},  \Hypotheses(\demonstrations^k \cup \{\demonstration\}))| \notag \\
      = ~&|\prefset_{\ordering,\hstar}( \hinit,  \Hypotheses(\demonstrations^k))|  \notag\\&~~~~~- \min_\demonstration |\prefset_{\ordering,\hstar}( \hinit,  \Hypotheses(\demonstrations^k \cup \{\demonstration\}))| \notag \\
      = ~&\max_\demonstration ( |\prefset_{\ordering,\hstar}( \hinit,  \Hypotheses(\demonstrations^k))|  \notag \\&~~~~~- |\prefset_{\ordering,\hstar}( \hinit,  \Hypotheses(\demonstrations^k \cup \{\demonstration\}))| ).
    \end{align}
  \item In this case, we have
    \begin{align*}
      &|\prefset_{\ordering,\hstar}( \hypothesis^{k+1}, \Hypotheses(\demonstrations^k \cup \{\demonstration^k\}))|\\
      &= |\prefset_{\ordering,\hstar}( \hypothesis_z, \Hypotheses(\demonstrations^k \cup \{\demonstration\}))|\\
      &= |\prefset_{\ordering,\hstar}( \hypothesis^k, \Hypotheses(\demonstrations^k \cup \{\demonstration'\}))|
    \end{align*}
    and the myopic algorithm picks $\demonstration^k=\demonstration'$. The learner does not move away from her current hypothesis: $\hypothesis^{k+1} = \hypothesis^k$. However, since $\Hypotheses(\{\demonstration'\})\setminus \Hypotheses(\{\demonstration\}) = \{\hypothesis^k\}$, we get
    \begin{align}
      &|\prefset_{\ordering,\hstar}( \hypothesis^{k+1},  \Hypotheses(\demonstrations^k \cup \{\demonstration^k\}))|\\
      &= |\prefset_{\ordering,\hstar}( \hypothesis^k, \Hypotheses(\demonstrations^k \cup \{\demonstration'\}))| \notag \\
      &= |\prefset_{\ordering,\hstar}( \hypothesis^k, \Hypotheses(\demonstrations^k \cup \{\demonstration\}))| + 1\notag \\
      &\stackrel{(a)}{=} \min_{\demonstration''} |\prefset_{\ordering,\hstar}( \hypothesis^k,  \Hypotheses(\demonstrations^k \cup \{\demonstration''\}))| + 1\notag \\
      &\stackrel{\eqref{eq:prefvschange}}{=} \min_{\demonstration''} |\prefset_{\ordering,\hstar}( \hinit,  \Hypotheses(\demonstrations^k \cup \{\demonstration''\}))| + 1, \notag % \label{eq:objective_rankgreedy}
    \end{align}
    where step (a) is due to the greedy choice of the myopic algorithm.
    Further note that before reaching $\hstar$, the gain of a greedy demonstration is positive. Therefore,
    \begin{align}
      \label{eq:greedy_case2}
      & |\prefset_{\ordering,\hstar}( \hypothesis^k,  \Hypotheses(\demonstrations^k))| - |\prefset_{\ordering,\hstar}( \hypothesis^{k+1},  \Hypotheses(\demonstrations^k \cup \{\demonstration^k\}))| \notag \\
      \geq ~&\frac12 ( 1 + |\prefset_{\ordering,\hstar}( \hypothesis^k,  \Hypotheses(\demonstrations^k))| -\notag \\&~~~~~~~~~~~~~~ |\prefset_{\ordering,\hstar}( \hypothesis^{k+1},  \Hypotheses(\demonstrations^k \cup \{\demonstration^k\}))| ) \notag \\
      = ~&\frac12 \max_\demonstration ( |\prefset_{\ordering,\hstar}( \hinit,  \Hypotheses(\demonstrations^k))|  -\notag \\&~~~~~~~~~~~~~~ |\prefset_{\ordering,\hstar}( \hinit,  \Hypotheses(\demonstrations^k \cup \{\demonstration\}))| ).
    \end{align} 
  \item In this case, $\demonstration^k$ is consistent with $\phi^k$, the greedy gain amounts to the maximal number of hypotheses removed from the preferred version space. Thus we have
    \begin{align}
      \label{eq:greedy_case3}
      & |\prefset_{\ordering,\hstar}( \hypothesis^k,  \Hypotheses(\demonstrations^k))| - |\prefset_{\ordering,\hstar}( \hypothesis^{k+1},  \Hypotheses(\demonstrations^k \cup \{\demonstration^k\}))| \notag \\      = ~&\max_\demonstration ( |\prefset_{\ordering,\hstar}( \hinit,  \Hypotheses(\demonstrations^k))|  -\notag \\&~~~~~~~~~~~~~~ |\prefset_{\ordering,\hstar}( \hinit,  \Hypotheses(\demonstrations^k \cup \{\demonstration\}))| ).
    \end{align}
  \end{enumerate}
  Combining Eq.~\eqref{eq:greedy_case1}, \eqref{eq:greedy_case2}, \eqref{eq:greedy_case3}, we have proven the inequality (\ref{lemma_ineq}).
% \end{proof}

% \paragraph{Proof of Theorem \ref{th_bound}}
% We are now ready to provide the proof for Theorem \ref{th_bound}.
% We show that under the uniform preference model, the cost of the greedy algorithm is within a logarithmic factor of the optimal cost.
% \begin{lemma} Let $\prefset \subseteq \Hypotheses$ be any subset of the version space s.t., $\hstar\in \prefset$. It holds that
%   $$\greedycost(\uordering, \hstar, \hypothesis, \prefset) \leq \left( \log{|\futurecostapprox(\hinit, \Hypotheses)|} + 1 \right)\optcost(\uordering, \hstar, \hypothesis, \prefset).$$
% \end{lemma}
% \begin{proof}
%   As discussed in \secref{sec:alg:optimal}, under uniform preference, the problem of teaching $\hstar$ reduces to the Set Cover problem, and the greedy heuristic $\futurecostapprox_u(\hypothesis, \prefset) = |\{\hypothesis' \in \prefset: \uorderingof{\hypothesis'}{\hypothesis} \leq \uorderingof{\hstar}{\hypothesis}\}| = |\prefset|$ reduces to the set cover objective, which is a monotone submodular function \cite{wolsey1982analysis}. The logarithmic approximation result then follows from Theorem 1 of \citet{wolsey1982analysis}.
% \end{proof}

% \begin{proof}[Proof of Theorem \ref{th_bound}]

  % Let $\prefset \subseteq \Hypotheses$ be any subset of the version space s.t., $\hstar\in \prefset$. It holds that
  % $$\greedycost(\uordering, \hstar, \hypothesis, \prefset) \leq \left( \log{|\futurecostapprox(\hinit, \Hypotheses)|} + 1 \right)\optcost(\uordering, \hstar, \hypothesis, \prefset).$$
  Based on the above discussions, we know that the teaching sequence provided by the myopic adaptive TLIP algorithm never adds new hypotheses into the initial preferred version space $\prefset_{\ordering,\hstar}( \hinit,  \Hypotheses(\demonstrations^k))$, and neither does it move \emph{consistent} hypotheses out of $\prefset_{\ordering,\hstar}( \hinit,  \Hypotheses(\demonstrations^k))$. The teaching objective thus reduces to a set cover objective, and the teaching finishes once all hypotheses, except $\hstar$, in the initial preferred version space are covered.

  From inequality (\ref{lemma_ineq}), we show that with each demonstration, the gain of the myopic adaptive TLIP algorithm is at least $\frac12$ the gain of the greedy set cover algorithm. Therefore, when the preference function is local but satisfies Conditions 1 and 2 of Theorem \ref{th_bound}, the myopic adaptive TLIP algorithm reduces to a 2-approximate greedy set cover algorithm for minimizing the AN teaching cost \cite{wolsey1982analysis}. So we have the logarithmic approximation results as follows.
	\[
    \begin{split}
    \textrm{AN-Cost}^{\textrm{WC}}_{\mathcal{S}}(\mathcal{T}^{\textrm{m}},& \mathcal{L}_{\sigma})\le\\
    &2(\log \vert\tilde{\Phi}_{\mathcal{S}}\vert+1)\textrm{AN-Complexity}_{\mathcal{S}}(\mathcal{L}_{\sigma}).
    \end{split}
    \] 
    
For minimizing the AL teaching cost, the myopic adaptive TLIP algorithm reduces to a 2-approximate greedy weighted set cover algorithm with the weights being the lengths of the demonstrations. So we have the logarithmic approximation results as follows.
	\[
    \begin{split}
    \textrm{AL-Cost}^{\textrm{WC}}_{\mathcal{S}}(\mathcal{T}^{\textrm{m}},& \mathcal{L}_{\sigma})\le\\
    &2(\log \vert\tilde{\Phi}_{\mathcal{S}}\vert+1)\textrm{AL-Complexity}_{\mathcal{S}}(\mathcal{L}_{\sigma}).
    \end{split}
    \]
    
\section{Case Study: Robot Navigation Scenario}
In this case study, we consider a robotic navigation scenario in a simulated space partitioned into 81 cells as shown in Fig.\ref{fig_map}. Each cell is associated with a color. A robot can be only at one cell at any time. The robot has five possible actions at each time step: stay still, go north, go south, go east or go west. For simplicity, we assume that each action is deterministic, i.e., there is zero slip rate. For example, when the robot takes action to go north at the cell located at (1, 1), it will land in the cell located at (2, 1). However, if the robot hit the boundaries, it will remain at the same position. 

As shown in the gridworld map, if the robot is at a red cell, then at the next time step, it can be at a red cell or a blue cell, but cannot be at a green cell or a yellow cell; and if the robot is at a green cell, then at the next time step, it can be at a green cell, a yellow cell or a blue cell, but cannot be at a red cell. We add such transition constraints to the Integer programming formulation in computing the demonstrations. 

\begin{figure}[th] 
	\centering
	\includegraphics[width=8cm]{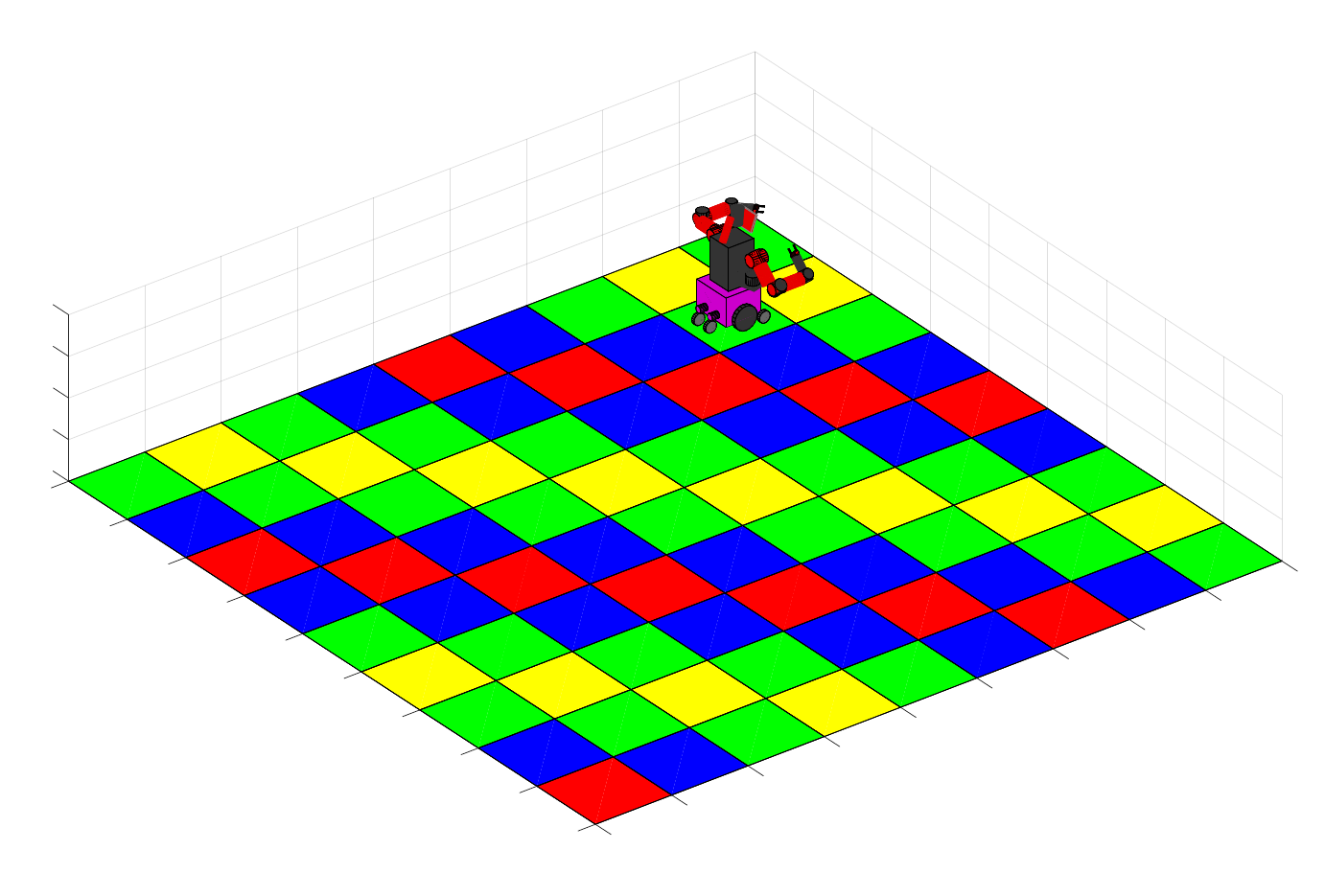}\caption{Simulated space in the robotic navigation scenario.}
	\label{fig_map}
\end{figure}

The hypothesis set of pLTL$_f$ formulas are listed in Table \ref{table_space}. The set $S$ of states is $\{$Red, Blue, Green, Yellow$\}$. In each teaching session, we randomly select both the initial hypothesis and the target hypothesis from the hypothesis set. All results are averaged over 10 teaching sessions.
% We set the ground-truth formula as ${\phi}^{\ast}=F_{\le3}(x\le5)$. 
  \begin{table}[h]
  	\begin{center}
  	
  		\caption{Hypothesis set of pLTL$_f$ formulas ($a=5, 10, 15$ correspond to hypothesis sets of sizes 90, 180 and 270, respectively). } \label{table_color}
%  		\rowcolors{2}{gray!25}{white}
        \scalebox{1}{
  		\begin{tabular}{|c|p{0.5\textwidth}|} 
  			\hline
  			Example & Hypothesis pLTL$_f$ Formulas \\
  			\hline
  			F-formulas &$F_{\le1}$Red, $F_{\le1}$Blue, $F_{\le1}$Green, $F_{\le1}$Yellow, \newline $F_{\le2}$Red, $F_{\le2}$Blue, $F_{\le2}$Green, $F_{\le2}$Yellow, \newline 
  			\dots \newline 
  			$F_{\le a}$Red, $F_{\le a}$Blue, $F_{\le a}$Green, $F_{\le a}$Yellow\\
  			\hline
  			G-formulas &$G_{\le1}$Red, $G_{\le1}$Blue, $G_{\le1}$Green, $G_{\le1}$Yellow, \newline $G_{\le2}$Red, $G_{\le2}$Blue, $G_{\le2}$Green, $G_{\le2}$Yellow, \newline 
  			\dots \newline 
  			$G_{\le a}$Red, $G_{\le a}$Blue, $G_{\le a}$Green, $G_{\le a}$Yellow\\
  			\hline
  		\end{tabular}
  		}
  	\end{center}
  \end{table}

\begin{figure}[!h]
	\centering
	\includegraphics[width=8cm]{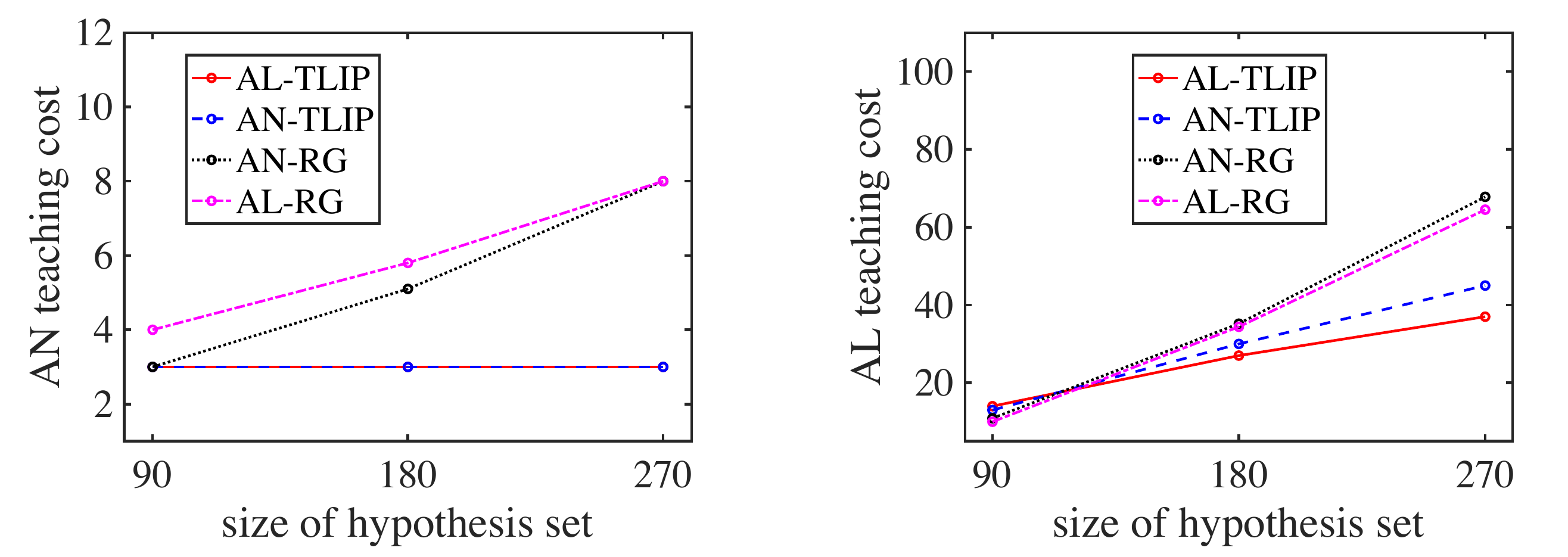}\caption{AN and AL teaching costs under global (uniform) preferences with increasing sizes of the hypothesis set in the robotic navigation scenario.}  
	\label{result1_robot}
\end{figure}		

\paragraph{Global preferences}We implement the following four methods for comparison for myopic teaching performances.
\begin{itemize}\denselist
    \item \textbf{AN-TLIP}: TLIP for minimizing AN teaching cost.
    \item \textbf{AL-TLIP}: TLIP for minimizing AL teaching cost.
    \item \textbf{AN-RG}: randomized greedy algorithm for minimizing AN teaching cost. At each iteration, we greedily pick the demonstration with minimal AN teaching cost among a randomly selected subset of demonstrations.   
    \item \textbf{AL-RG}: same with AN-RG except that here we minimize the AL teaching cost.
\end{itemize}

\figref{result1_robot} shows that AN-TLIP and AL-TLIP are the best for minimizing the AN teaching costs and the AL teaching costs, respectively. The AN teaching costs using AN-TLIP are the same as the AN teaching costs using AL-TLIP, which are up to 62.5\% less than those using AN-RG and AL-RG. The AL teaching costs using AL-TLIP are 17.78\%, 45.43\% and 42.64\% less than those using AN-TLIP, AN-RG and AL-RG, respectively.

\paragraph{Local preferences} We use the mapping $\varrho$ that maps each color to an integer. Specifically, $\varrho($Red$)=1$, $\varrho($Blue$)=2$, $\varrho($Green$)=3$, and $\varrho($Yellow$)=4$. For two pLTL$_f$ formulas $\phi_1=F_{\le i_1}$Red and $\phi_2=F_{\le i_2}$Green, we define the Manhattan distance between $\phi_1$ and $\phi_2$ as $\vert i_1-i_2\vert + \vert \varrho(\textrm{Red})-\varrho(\textrm{Green})\vert$. We consider the following local
preference:\\
(1) The learner prefers formulas with the same temporal operator as that in the learner's previous hypothesis;\\
(2) With the same temporal operator the learner prefers formulas that are ``closer'' to the learner's previous hypothesis in terms of the Manhattan distance. 

	\begin{figure}[th]
		\centering
		\includegraphics[width=8cm]{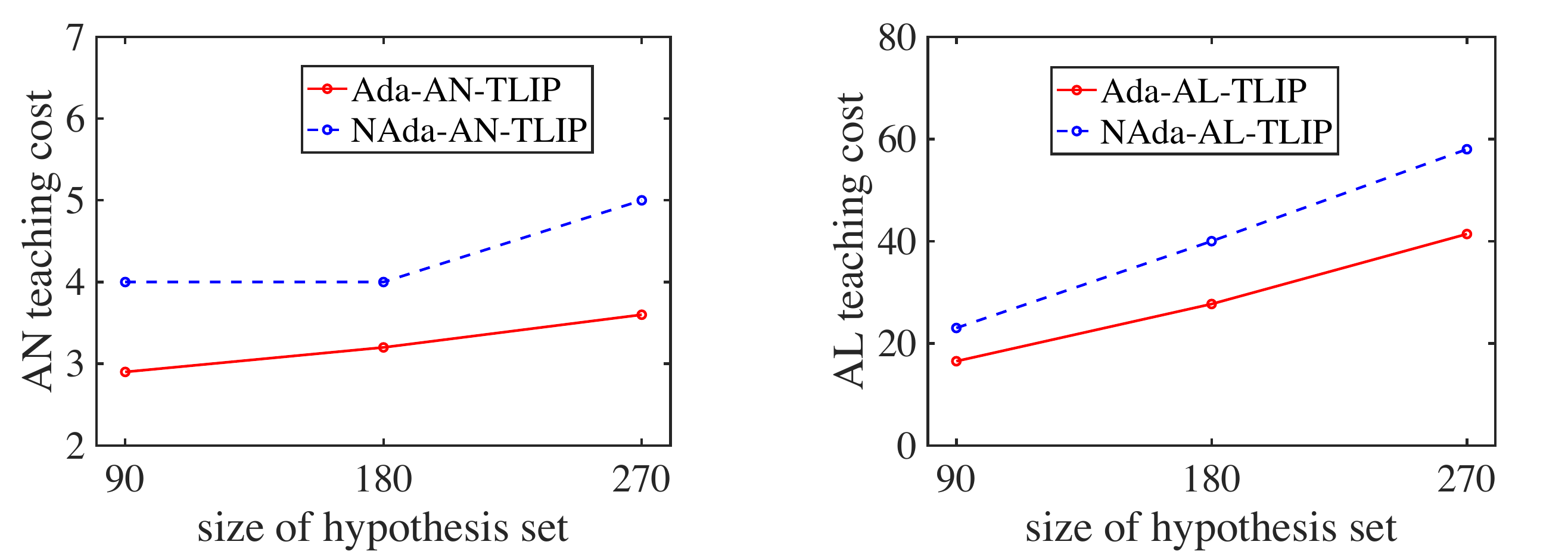}\caption{AN and AL teaching costs for adaptive and non-adaptive teaching with increasing sizes of the hypothesis set in the robotic navigation scenario.} 
		\label{result3}
	\end{figure}
To test the advantage of adaptive teaching, we compare it with non-adaptive teaching in the presence of uncertainty. We consider local preferences with added uncertainty noises. Specifically, the learner has equal preference of selecting the formulas in the version space that have the least Manhattan distance from the current hypothesis and also any formula that can be perturbed from these formulas in the version space (here ``perturb'' means adding or subtracting the parameters $i$ or $\varrho(\cdot)$ by 1).

\figref{result3} shows that Ada-AN-TLIP (i.e., adaptive AN-TLIP) can reduce the AN teaching costs by 28\% compared with NAda-AN-TLIP (i.e., non-adaptive AN-TLIP), and Ada-AL-TLIP (i.e., adaptive AL-TLIP) can reduce the AN teaching costs by 30.75\% compared with NAda-AL-TLIP (i.e., non-adaptive AL-TLIP).

\end{document}